\documentclass[runningheads]{llncs}
\pdfoutput=1
\usepackage{bibnames}
\usepackage{microtype}
\usepackage{booktabs} 
\usepackage{epsf}
\usepackage{epsfig}
\usepackage{subfigure}
\usepackage{xspace}
\usepackage{hyperref}
\usepackage{listings}
\usepackage{xcolor}
\usepackage{enumitem}
\usepackage{amssymb}
\usepackage{amsmath}
\usepackage{mathtools}
\usepackage{etextools}
\usepackage{algorithm}
\usepackage{algorithmic}
\usepackage{graphicx}
\allowdisplaybreaks

\begin{document}
\title{SLSGD: Secure and Efficient Distributed On-device Machine Learning}
\titlerunning{SLSGD: Secure and Efficient Distributed On-device Machine Learning}
%
\author{Cong Xie\inst{1}\and
Sanmi Koyejo\inst{1}\and
Indranil Gupta\inst{1}}
%
%
\institute{University of Illinois at Urbana-Champaign
\email{\{cx2,sanmi,indy\}@illinois.edu}}
\maketitle              

\def\Blue{\color{blue}}
\def\Purple{\color{purple}}

\def\E{{\mathbb{E}}}

\def\R{{\mathbb{R}}}

\def\Ncal{\mathcal{N}}
\def\argmax{\mathop{\rm argmax}}
\def\argmin{\mathop{\rm argmin}}
\providecommand{\abs}[1]{\lvert#1\rvert}
\providecommand{\norm}[2]{\lVert#1\rVert_{#2}}

\def\fold{\mathrm{fold}}
\def\index{\mathrm{index}}
\def\sgn{\mathrm{sgn}}
\def\tr{\mathrm{tr}}
\def\rk{\mathrm{rank}}
\def\diag{\mathsf{diag}}
\def\const{\mathrm{Const}}
\def\dg{\mathsf{dg}}
\def\st{\mathsf{s.t.}}
\def\etal{{\em et al.\/}\,}
\def\prox{\mathrm{prox}^h_\gamma}
\newcommand{\indep}{{\;\bot\!\!\!\!\!\!\bot\;}}

\newcommand{\mtrxt}[1]{{#1}^\top}
\newcommand{\mtrx}[4]{\left[\begin{matrix}#1 & #2 \\ #3 & #4\end{matrix}\right]}
\DeclarePairedDelimiter\vnorm{\lVert}{\rVert}
\DeclarePairedDelimiterX{\innerprod}[2]{\langle}{\rangle}{#1, #2}

\newcommand{\twopartdef}[4]
{
	\left\{
		\begin{array}{ll}
			#1 & \mbox{if } #2 \\
			#3 & \mbox{if } #4
		\end{array}
	\right.
}

\newcommand{\tabincell}[2]{\begin{tabular}{@{}#1@{}}#2\end{tabular}}

\DeclarePairedDelimiter\ceil{\lceil}{\rceil}
\DeclarePairedDelimiter\floor{\lfloor}{\rfloor}

\newcommand{\ip}[2]{\left\langle #1, #2 \right \rangle}

\def\trmean{{\tt Trmean}}
\def\aggr{{\tt Aggr}}

\spnewtheorem{assumption}[note]{Assumption}{\bfseries}{\itshape}

\begin{abstract}
We consider distributed on-device learning with limited communication and security requirements. We propose a new robust distributed optimization algorithm with efficient communication and attack tolerance. The proposed algorithm has provable convergence and robustness under non-IID settings. Empirical results show that the proposed algorithm stabilizes the convergence and tolerates data poisoning on a small number of workers.
\end{abstract}

\section{Introduction}

Edge devices/IoT such as smart phones, wearable devices, sensors, and smart homes are
increasingly generating massive, diverse, and private data.
In response, there is a trend towards moving computation, including the training of machine-learning models, from cloud/datacenters to edge devices~\cite{anguita2013public,pantelopoulos2010survey}. Ideally, since trained on massive representative data, the resulting models exhibit improved generalization.
In this paper, we consider distributed on-device machine learning. The distributed system is a server-worker architecture. The workers are placed on edge devices, which train the models on the private data. The servers are placed on the cloud/datacenters which maintain a shared global model. Distributed settings require addressing some novel engineering challenges, including the following:
\setitemize[0]{leftmargin=*}
\begin{itemize}
\item \textbf{Limited, heterogeneous computation.} Edge devices, including smart phones, wearable devices, sensors, or vehicles typically have weaker computational ability, compared to the workstations or datacenters used in typical distributed machine learning. Thus, simpler models and stochastic training are usually applied in practice. Furthermore, different devices have different computation capabilities. 
\item \textbf{Limited communication} The connection to the central servers are not guaranteed. Communication can be frequently unavailable, slow, or expensive (in money or in the power of battery). Thus, frequent high-speed communication is typically unaffordable.
\item \textbf{Decentralized, non-IID training data.} Privacy needs and legal requirements (e.g., US HIPAA laws~\cite{act1996health} in a smart hospital, or Europe's GDPR law~\cite{GDPR}) may necessitate that training be performed on-premises using IoT devices and edge machines, and that data and models must not be deposited in the cloud or cloudlets. In more general cases, the users simply dislike sharing their on-device data which potentially reveals private information.
As a result, the data distribution on different devices are neither mixed nor IID i.e. unlike standard settings, device data are non-identically distributed samples from the population. 
This is particularly true when each device is controlled by a specific user whose behavior is supposed to be unique. Furthermore, the sampled data on nearby devices are potentially non-independent, since such devices can be shared by the same user or family. For example, the data of a step counter from a wearable fitness tracker and a smart phone owned by the same user can have different distributions of motion data with mutual dependency. Imagine that the fitness tracker is only used when the user is running, and the smart phone is only used when the user is walking, which results in different distributions. On the other hand, the complementation yields dependency. 
\item \textbf{Untrusted workers and data poisoning.} The servers have limited control over the users' behavior. To protect the privacy, the users are in general anonymous to the servers. Although it is possible to verity the identity of workers/devices~\cite{attestation}, nefarious users can feed poisoned data with abnormal behaviors without backdooring OS. As a result, some workers may push models learned on poisoned data to the servers. 
\end{itemize}

To overcome the challenges above, we introduce \underline{S}ecure \underline{L}ocal \underline{S}tochastic \underline{G}radient \underline{D}escent~(SLSGD), which reduces the communication overhead with local updates, and secures the global model against nefarious users and poisoned data. We summarize the key properties of SLSGD below:
\setitemize[0]{leftmargin=*}
\begin{itemize}
\item \textbf{Local SGD.} SGD is widely used for training models with lower computation overhead. To reduce communication overhead, we use SGD with local updates. The workers do not synchronize with the server after each local gradient descent step. After several local iterations, the workers push the updated model to the servers, which is different from the traditional distributed synchronous SGD where gradients are pushed in each local gradient descent step. To further reduce the communication overhead, the training tasks are activated on a random subset of workers in each global epoch.
\item \textbf{Secure aggregation.} In each global epoch, the servers send the latest global model to the activated workers, and aggregate the updated local models. In such procedure, there are two types of threats: i) poisoned models pushed from comprised devices, occupied or hacked by nefarious users; ii) accumulative error, variance, or models over-fitted on the local dataset, caused by infrequent synchronization of local SGD. To secure the global model against these two threats, we use robust aggregation which tolerates abnormal models, and moving average which mitigates the errors caused by infrequent synchronization. 
\end{itemize}


To our knowledge, there is limited work on local SGD with theoretical guarantees~\cite{yu2018parallel,stich2018local}. The existing convergence guarantees are based on the strong assumption of IID training data or homogeneous local iterations, which we have argued is inappropriate for distributed learning on edge devices. 

We propose SLSGD, which is a variant of local SGD with provable convergence under non-IID and heterogeneous settings, and tolerance to nefarious users. 
In summary, the main contributions are listed as follows:
\setitemize[0]{leftmargin=*}
\begin{itemize} 
\item We show that SLSGD theoretically converges to global optimums for strongly convex functions, non-strongly convex functions, and a restricted family of non-convex functions, under non-IID settings. Furthermore, more local iterations accelerate the convergence. 
\item We show that SLSGD tolerates a small number of workers training on poisoned data. As far as we know, this paper is the first to investigate the robustness of local SGD. 
\item We show empirically that the proposed algorithm stabilizes the convergence, and protects the global model from data poisoning.
\end{itemize}

\section{Related Work}

Our algorithm is based on local SGD introduced in \cite{yu2018parallel,stich2018local}. The major differences are:
\begin{enumerate}
\item We assume non-IID training data and heterogeneous local iterations among the workers. In previous work, local SGD and its convergence analysis required IID training data, or same number of local iterations within each global epoch~(or both). However, these assumptions are unreasonable for edge computing, due to privacy preservation and heterogeneous computation capability.
\item Instead of using the averaged model to overwrite the current global model on the server, we take robust aggregation, and use a moving average to update the current model. These techniques not only secure the global model against data poisoning, but also mitigate the error caused by infrequent synchronization of local SGD.
\end{enumerate}

The limited communication power of edge devices also motivates federated learning~\cite{konevcny2015federated,konevcny2016federated,mcmahan2016communication}, whose algorithm is similar to local SGD, and scenario is similar to our non-IID and heterogeneous settings. Unfortuntaely, federated learning lacks provable convergence guarantees. Furthermore, the issues of data poisoning have not been addressed in previous work. To the best of our knowledge, our proposed work is the first that considers both convergence and robustness, theoretically and practically, on non-IID training data.

Similar to the traditional distributed machine learning, we use the server-worker architecture, which is similar to the Parameter Server~(PS) architecture. Stochastic Gradient Descent~(SGD) with PS architecture, 
is widely used in typical distributed machine learning~\cite{li2014scaling,ho2013more,li2014communication}. Compared to the traditional distributed learning on PS , SLSGD has much less synchronization. Furthermore, in SLSGD, the workers push trained models instead of gradients to the servers.
 
Approaches based on robust statistics are often used to address security issues in the PS architecture~\cite{yin2018byzantine,xie2018phocas}. This enables procedures which tolerate multiple types of attacks and system failures. However, the existing methods and theoretical analysis do not consider local training on non-IID data. So far, the convergence guarantees are based on robust gradient aggregation. In this paper, we provide convergence guarantees for robust model aggregation. Note that gradients and models~(parameters) have different properties. For example, the gradients converge to $0$ for unconstrained problems, while the models do not have such property. On the other hand, recent work has considered attacks targeting federated learning~\cite{bagdasaryan2018backdoor,fung2018mitigating,bhagoji2018analyzing}, but do not propose defense techniques with provable convergence. 

There is growing literature on the practical applications of edge and fog computing~\cite{garcia2015edge,hong2013mobile} in various scenarios such as smart home or sensor networks. More and more big-data applications are moving from the cloud to the edge, including for machine-learning tasks~\cite{cao2015distributed,mahdavinejad2018machine,zeydan2016big}. Although computational power is growing, edge devices are still much weaker than the workstations and datacenters used in typical distributed machine learning e.g. due to the limited computation and communication capacity, and limited power of batteries. To this end, there are machine-learning frameworks with simple architectures such as MobileNet~\cite{howard2017mobilenets} which are designed for learning with weak devices.

\section{Problem Formulation}

Consider distributed learning with $n$ devices. On each device, there is a worker process that trains the model on local data. The overall goal is to train a global model $x \in \R^d$ using data from all the devices.

To do so, we consider the following optimization problem:
\begin{align*}
\min_{x \in \R^d} F(x), 
\end{align*}
where $F(x) = \frac{1}{n} \sum_{i \in [n]} \E_{z^i \sim \mathcal{D}^i} f(x; z^i)$, for $\forall i \in [n]$, $z^i$ is sampled from the local data $\mathcal{D}^i$ on the $i$th device.

\begin{table}[htb]
\caption{Notations and Terminologies}
\label{tbl:notations}
\begin{center}
\begin{small}
\begin{tabular}{|l|l|}
\hline 
Notation/Term  & Description \\ \hline
$n$    & Number of devices \\ \hline
$k$    & Number of simutaneously updating devices  \\ \hline
$T$    & Number of communication epochs  \\ \hline
$[n]$    & Set of integers $\{1, \ldots, n \}$  \\ \hline
$S_t$    & Randomly selected devices in the $t^{\mbox{th}}$ epoch  \\ \hline
$b$    & Parameter of trimmed mean  \\ \hline
$H_{min}$    & Minimal number of local iterations  \\ \hline
$H^i_t$    & Number of local iterations in the $t^{\mbox{th}}$ epoch \\ & on the $i$th device \\ \hline
$x_t$    & Initial model in the $t^{\mbox{th}}$ epoch  \\ \hline
$x^i_{t,h}$    & Model updated in the $t^{\mbox{th}}$ epoch,  $h$th local iteration, on the $i$th device  \\ \hline
$\mathcal{D}^i$    & Dataset on the $i$th device  \\ \hline
$z^i_{t, h}$    & Data~(minibatch) sampled in the $t^{\mbox{th}}$ epoch, \\ & $h$th local iteration, on the $i$th device  \\ \hline
$\gamma$    & Learning rate  \\ \hline
$\alpha$    & Weight of moving average  \\ \hline
$\| \cdot \|$    & All the norms in this paper are $l_2$-norms  \\ \hline
Device    & Where the training data are placed  \\ \hline
Worker    & One worker on each device,  process that trains the model  \\ \hline
User    & Agent that produces data on the devices, and/or controls the devices  \\ \hline
Nefarious user    & Special user that produces poisoned data or has abnormal behaviors  \\ \hline
\end{tabular}
\end{small}
\end{center}
\end{table}

\subsection{Non-IID Local Datasets}

Note that different devices have different local datasets, i.e., $\mathcal{D}^i \neq \mathcal{D}^j, \forall i \neq j$. Thus, samples drawn from different devices have different expectations, which means that $\E_{z^i \sim \mathcal{D}^i} f(x; z^i) \neq \E_{z^j \sim \mathcal{D}^j} f(x; z^j), \forall i \neq j$. Further, since different devices can be possessed by the same user or the same group of users~(e.g., families), samples drawn from different devices can be potentially dependent on each other.

\subsection{Data Poisoning}

The users are anonymous to the servers. Furthermore, it is impossible for the servers to verify the benignity of the on-device training data. Thus, the servers can not trust the edge devices. A small number of devices may be susceptible to data poisoned by abnormal user behaviors or in the worst case, are controlled by users or agents who intend to directly upload harmful  models to the servers. 

In this paper, we consider a generalized threat model, where the workers can push arbitrarily bad models to the servers. The bad models can cause divergence of training. Beyond more benign issues such as hardware, software or communication failures, there are multiple ways for nefarious users to manipulate the uploaded models e.g. data poisoning~\cite{bae2018security}. In worst case, nefarious users can even directly hack the devices and replace the correct models with arbitrary values. We provide a more formal definition of the threat model in Section~\ref{subsect:threat}.

\section{Methodology}

In this paper, we propose SLSGD: SGD with communication efficient local updates and secure model aggregation. A single execution of SLSGD is composed of $T$ communication epochs. At the beginning of each epoch, a randomly selected group of devices $S_t$ pull the latest global model from the central server. Then, the same group of devices locally update the model without communication with the central server. At the end of each epoch, the central server aggregates the updated models and then updates the global model.

In the $t^{\mbox{th}}$ epoch, on the $i$th device, we locally solve the following optimization problem using SGD for $H^{i}_t$ iterations:
\begin{align*}
\min_{x \in \R^d} \E_{z^i \sim \mathcal{D}^i} f(x; z^i).
\end{align*}
Then, the server collects the resulting local models $x^i_{t, H^i_t}$, and aggregates them using $\aggr\left( \{x^i_{t, H^i_t}: i \in S_t\} \right)$. Finally, we update the model with a moving average over the current model and the aggregated local models.

The detailed algorithm is shown in Algorithm~\ref{alg:robust_fed}. $x_{t, h}^i$ is the model parameter updated in $h$th local iteration of the $t^{\mbox{th}}$ epoch, on the $i$th device. $z_{t, h}^i$ is the data randomly drawn in $h$th local iteration of the $t^{\mbox{th}}$ epoch, on the $i$th device. $H_t^i$ is the number of local iterations in the $t^{\mbox{th}}$ epoch, on the $i$th device. $\gamma$ is the learning rate and $T$ is the total number of epochs. Note that if we take Option I~(or Option II with $b=0$) with $\alpha=1$, the algorithm is the same as the federated learning algorithm \textit{FedAvg}~\cite{mcmahan2016communication}. Furthermore, if we take homogeneous local iterations $H^i_t = H, \forall i$, Option I with $\alpha=1$ is the same as local SGD~\cite{stich2018local}. Thus, FedAvg and local SGD are both special cases of SLSGD.

\begin{algorithm}[hbt]
\caption{SLSGD}
\begin{algorithmic}[1]
\STATE Input: $k \in [n]$, $b$
\STATE Initialize $x_0$
\FORALL{epoch $t \in [T]$}
	\STATE Randomly select a group of $k$ workers, denoted as $S_t \subseteq [n]$
	\FORALL{$i \in S_t$ in parallel}
		\STATE Receive the latest global model $x_{t-1}$ from the server
		\STATE $x_{t, 0}^i \leftarrow x_{t-1}$
		\FORALL{local iteration $h \in [H_t^i]$}
			\STATE Randomly sample $z_{t, h}^i$
			\STATE $x_{t, h}^i \leftarrow x_{t, h-1}^i - \gamma \nabla f(x_{t, h-1}^i; z_{t, h}^i)$
		\ENDFOR
		\STATE Push $x_{t, H_t^i}^i$ to the server 
	\ENDFOR
	\STATE Aggregate:
	$
	x'_{t} \leftarrow
	\begin{cases}
		 \mbox{Option I: } & \frac{1}{k} \sum_{i \in S_t} x_{t, H_t^i}^i \\
		 \mbox{Option II: } & \trmean_b \left( \left\{ x_{t, H_t^i}^i: i \in S_t \right\} \right)
	\end{cases}
	$
	\STATE Update the global model: $x_{t} \leftarrow (1-\alpha) x_{t-1} + \alpha x'_{t}$
\ENDFOR
\end{algorithmic}
\label{alg:robust_fed}
\end{algorithm} 

\subsection{Threat Model and Defense Technique}
\label{subsect:threat}

First, we formally define the threat model.

\begin{definition}(Threat Model)
\label{def:threat}
In Line~12 of Algorithm~\ref{alg:robust_fed}, instead of the correct $x_{t, H_t^i}^i$, a worker, training on poisoned data or controlled by an abnormal/nefarious user,
may push arbitrary values to the server.
\end{definition}

\begin{remark}
Note that the users/workers are anonymous to the servers, and the nefarious users can sometimes pretend to be well-behaved to fool the servers. Hence, it is impossible to surely identify the workers training on poisoned data, according to their historical behavior. 
\end{remark}

In Algorithm~\ref{alg:robust_fed}, Option II uses the trimmed mean as a robust aggregation which tolerates the proposed threat model.
To define the trimmed mean, we first define the order statistics.
\begin{definition}(Order Statistics)
\label{def:ord_stat}
By sorting the scalar sequence $\{u_i: i \in [k], u_i \in \R\}$, we get $u_{1:k} \leq u_{2:k} \leq \ldots \leq u_{k:k}$, where $u_{i:k}$ is the $i$th smallest element in $\{u_i: i \in [k]\}$.
\end{definition}
Then, we define the trimmed mean.
\begin{definition}(Trimmed Mean)
\label{def:trim}
For $b \in \{0, 1, \ldots, \lceil k/2 \rceil - 1 \}$, the $b$-trimmed mean of the set of scalars $\{u_i: i \in [k]\}$ is defined as follows:
\[
\trmean_b(\{u_i: i \in [k]\}) = \frac{1}{k-2b} \sum_{i=b+1}^{k-b} u_{i:k},
\] where  $u_{i:k}$ is the $i$th smallest element in $\{u_i: i \in [i]\}$ defined in Definition~\ref{def:ord_stat}. The high-dimensional version~($u_i \in \R^d$) of $\trmean_b(\cdot)$ simply applies the trimmed mean in a coordinate-wise manner.
\end{definition}

Note that the trimmed mean (Option II) is equivalent to the standard mean (Option I) if we take $b=0$.

\begin{remark}
Algorithm~\ref{alg:robust_fed} provides two levels of defense: robust aggregation~(Line 14) and moving average~(Line 15). The robust aggregation tries to filter out the models trained on poisoned data. The moving average mitigates not only the extra variance/error caused by robust aggregation and data poisoning, but also the accumulative error caused by infrequent synchronization of local updates. 
\end{remark}

\begin{remark}
We can also replace the coordinate-wise trimmed mean with other robust statistics such as geometric median~\cite{chen2019distributed}. We choose coordinate-wise median/trimmed mean in this paper because unlike geometric median, trimmed mean has a computationally efficient closed-form solution.
\end{remark}

\section{Convergence Analysis}
In this section, we prove the convergence of Algorithm~\ref{alg:robust_fed} with non-IID data, for a restricted family of non-convex functions. Furthermore, we show that the proposed algorithm tolerates the threat model introduced in Definition~\ref{def:threat}. We start with the assumptions required by the convergence guarantees. 

\subsection{Assumptions} 
For convenience, we denote
$
F^i(x) = \E_{z^i \sim \mathcal{D}^i} f(x; z^i).
$
\begin{assumption} (Existence of Global Optimum)
\label{asm:loss}
We assume that there exists at least one (potentially non-unique) global minimum of the loss function $F(x)$, denoted by $x^*$.
\end{assumption}
\begin{assumption} (Bounded Taylor's Approximation)
\label{asm:smooth_cvx}
We assume that for $\forall x, z$, $f(x; z)$ has $L$-smoothness and $\mu$-lower-bounded Taylor's approximation:
\begin{align*}
&\ip{\nabla f(x; z)}{y-x} + \frac{\mu}{2} \|y-x\|^2 \leq f(y;z) - f(x;z) \\
&\leq \ip{\nabla f(x; z)}{y-x} + \frac{L}{2} \|y-x\|^2,
\end{align*}
where $\mu \leq L$, and $L > 0$.
\end{assumption}
Note that Assumption~\ref{asm:smooth_cvx} covers the case of non-convexity by taking $\mu < 0$, non-strong convexity by taking $\mu = 0$, and strong convexity by taking $\mu > 0$. 
\begin{assumption} (Bounded Gradient)
\label{asm:variance}
We assume that for $\forall x \in \R^d, i \in [n]$, and $\forall z \sim \mathcal{D}^i$, we have $\| \nabla f(x; z) \|^2 \leq V_1$.
\end{assumption}

Based on the assumptions above, we have the following convergence guarantees. All the detailed proofs can be found in the appendix. 

\subsection{Convergence without Data Poisoning}


First, we analyze the convergence of Algorithm~\ref{alg:robust_fed} with Option I, where there are no poisoned workers.
\begin{theorem}
\label{thm:convergence}
We take $\gamma \leq \min \left(\frac{1}{L}, 2\right)$. After $T$ epochs, Algorithm~\ref{alg:robust_fed} with Option I converges to a global optimum:
\begin{align*}
    &\E \left[ F(x_T) - F(x_*) \right] 
    \leq \left( 1-\alpha + \alpha (1 - \frac{\gamma}{2})^{H_{min}} \right)^T \left[ F(x_0) - F(x_*) \right] \\
    &\quad + \left[ 1 -  \left( 1-\alpha + \alpha (1 - \frac{\gamma}{2})^{H_{min}} \right)^T \right] \mathcal{O}\left( V_1 + \left( 1 + \frac{1}{k} - \frac{1}{n} \right) V_2 \right),
\end{align*}
where $V_2 = \max_{t \in \{0, T-1\}, h \in \{0, H^i_t - 1\}, i \in [n]} \|x_{t, h}^i - x_*\|^2$.
\end{theorem}
\begin{remark}
When $\alpha \rightarrow 1$, $\left( 1-\alpha + \alpha (1 - \frac{\gamma}{2})^{H_{min}} \right)^T \rightarrow (1 - \frac{\gamma}{2})^{T H_{min}}$, which results in nearly linear convergence to the global optimum, with error $\mathcal{O}(V_1 + V_2)$. When $\alpha \rightarrow 0$, the error is nearly reduced $0$, but the convergence will slow down. We can tune $\alpha$ to trade-off between the convergence rate and the error. In practice, we can take diminishing $\alpha$: $\alpha_t \propto \frac{1}{t^2}$, where $\alpha_t$ is the $\alpha$ in the $t^{\mbox{th}}$ global epoch. Furthermore, taking $\alpha_T = \frac{1}{T^2}$, $\lim_{T \rightarrow +\infty} \left[ 1 -  \left( 1-\alpha_T + \alpha_T (1 - \frac{\gamma}{2})^{H_{min}} \right)^T \right] = 0$.
\end{remark}


\subsection{Convergence with Data Poisoning}

Under the threat model defined in Definition~\ref{def:threat}, in worst case, Algorithm~\ref{alg:robust_fed} with Option I and $\alpha=1$~(local SGD) suffers from unbounded error.
\begin{proposition} (Informal)
Algorithm~\ref{alg:robust_fed} with Option I and $\alpha = 1$ can not tolerate the threat model defined in Definition~\ref{def:threat}.
\end{proposition}
\begin{proof} (Sketch)
Without loss of generality, assume that in a specific epoch $t$, among all the $k$ workers, the last $q_1$ of them are poisoned. For the poisoned workers, instead of pushing the correct value $x_{t, H_t^i}^i$ to the server, they push $- \frac{k-q_1}{q_1} x_{t, H_t^i}^i + c$, where $c$ is an arbitrary constant. For convenience, we assume IID~(required by local SGD, but not our algorithm) local datasets for all the workers. Thus, the expectation of the aggregated global model becomes $\frac{1}{k} \left\{ (k-q_1) \E\left[ x_{t, H_t^i}^i \right] + q_1 \E\left[ - \frac{k-q_1}{q_1} x_{t, H_t^i}^i + c \right] \right\} = \frac{q_1}{k} c$, which means that in expectation, the aggregated global model can be manipulated to take arbitrary values, which results in unbounded error.
\end{proof}
In the following theorems, we show that using Algorithm~\ref{alg:robust_fed} with Option II, the error can be upper bounded. 
\begin{theorem}
Assume that additional to the $n$ normal workers, there are $q$ workers training on poisoned data, where $q \ll n$, and $2q \leq 2b < k$. 
We take $\gamma \leq \min \left(\frac{1}{L}, 2\right)$. After $T$ epochs, Algorithm~\ref{alg:robust_fed} with Option II converges to a global optimum:
\begin{align*}
    &\E \left[ F(x_T) - F(x_*) \right] 
    \leq \left( 1-\alpha + \alpha (1 - \frac{\gamma}{2})^{H_{min}} \right)^T \left[ F(x_0) - F(x_*) \right] \\
    & \quad + \left[ 1 -  \left( 1-\alpha + \alpha (1 - \frac{\gamma}{2})^{H_{min}} \right)^T \right] \left[ \mathcal{O}(V_1) +  \mathcal{O}(\beta V_2)  \right],
\end{align*}
where $V_2 = \max_{t \in \{0, T-1\}, h \in \{0, H^i_t - 1\}, i \in [n]} \|x_{t, h}^i - x_*\|^2$, $\beta = 1 + \frac{1}{k-q} - \frac{1}{n} + \frac{k (k+b)}{(k-b-q)^2}$.
\end{theorem}
\begin{remark}
\label{rem:convergence_subset_robust}
Note that the additional error caused by the $q$ poisoned workers and $b$-trimmed mean is controlled by the factor $\frac{k(k+b)}{(k-b-q)^2}$, which decreases when $q$ and $b$ decreases, or $k$ increases.
\end{remark}

\section{Experiments}

In this section, we evaluate the proposed algorithm by testing its convergence and robustness. Note that zoomed figures of the empirical results can be found in the appendix.

\subsection{Datasets and Evaluation Metrics}
We conduct experiments on the benchmark CIFAR-10 image classification dataset~\cite{krizhevsky2009learning}, which is composed of 50k images for training and 10k images for testing. Each image is resized and cropped to the shape of $(24,24,3)$. We use a convolutional neural network~(CNN) with 4 convolutional layers followed by 1 fully connected layer. We use a simple network architecture, so that it can be easily handled by edge devices.  The detailed network architecture can be found in our submitted source code (will also be released upon publication).
The experiments are conducted on CPU devices. We implement SLSGD using the MXNET~\cite{chen2015mxnet} framework.

We also conduct experiments of LSTM-based language models on WikiText-2 dataset~\cite{merity2016pointer}. The model architecture was taken from the MXNET and Gluon-NLP tutorial~\cite{gluonnlp}. The results can be found in the appendix.

In each experiment, the training set is partitioned onto $n=100$ devices. We test the preformance of SLSGD on both balanced and unbalanced partitions:
\setitemize[0]{leftmargin=*}
\begin{itemize}
\item \textbf{Balanced Partition.} Each of the $n=100$ partitions has $500$ images.
\item \textbf{Unbalanced Partition.} To make the setting more realistic, we partition the training set into unbalanced sizes. The sizes of the $100$ partitions are $104, 112, \ldots, 896$~(an arithmetic sequence with step $8$, starting with $104$). Furthermore, to enlarge the variance, we make sure that in each partition, there are at most $5$ different labels out of all the $10$ labels. Note that some partitions only have one label.
\end{itemize}

In each epoch, $k=10$ devices are randomly selected to launch local updates, with the minibatch size of $50$. We repeat each experiment 10 times and take the average. We use top-1 accuracy on the testing set, and cross entropy loss function on the training set as the evaluation metrics.

The baseline algorithm is \textit{FedAvg} introduced by \cite{mcmahan2016communication}, which is a special case of our proposed Algorithm~\ref{alg:robust_fed} with Option I and $\alpha=1$. To make the comparison clearer, we refer to \textit{FedAvg} as ``SLSGD, $\alpha=1, b=0$''.

We test SLSGD with different hyperparameters $\gamma$, $\alpha$, and $b$~(definitions can be found in Table~\ref{tbl:notations}).

\begin{figure*}[htb!]
\centering
\subfigure[Top-1 accuracy on testing set]{\includegraphics[width=0.49\textwidth,height=3.8cm]{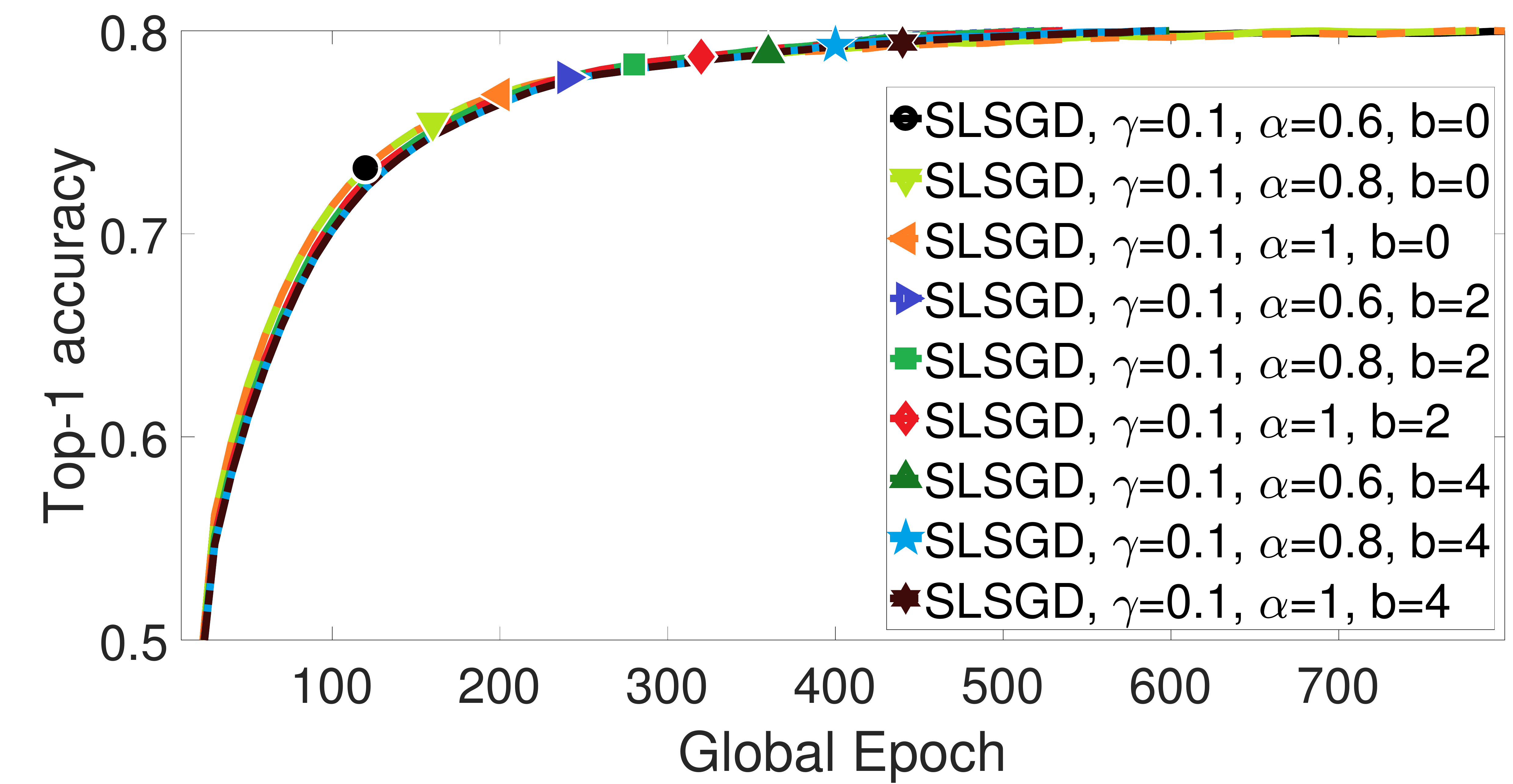}}
\subfigure[Cross entropy on training set]{\includegraphics[width=0.49\textwidth,height=3.8cm]{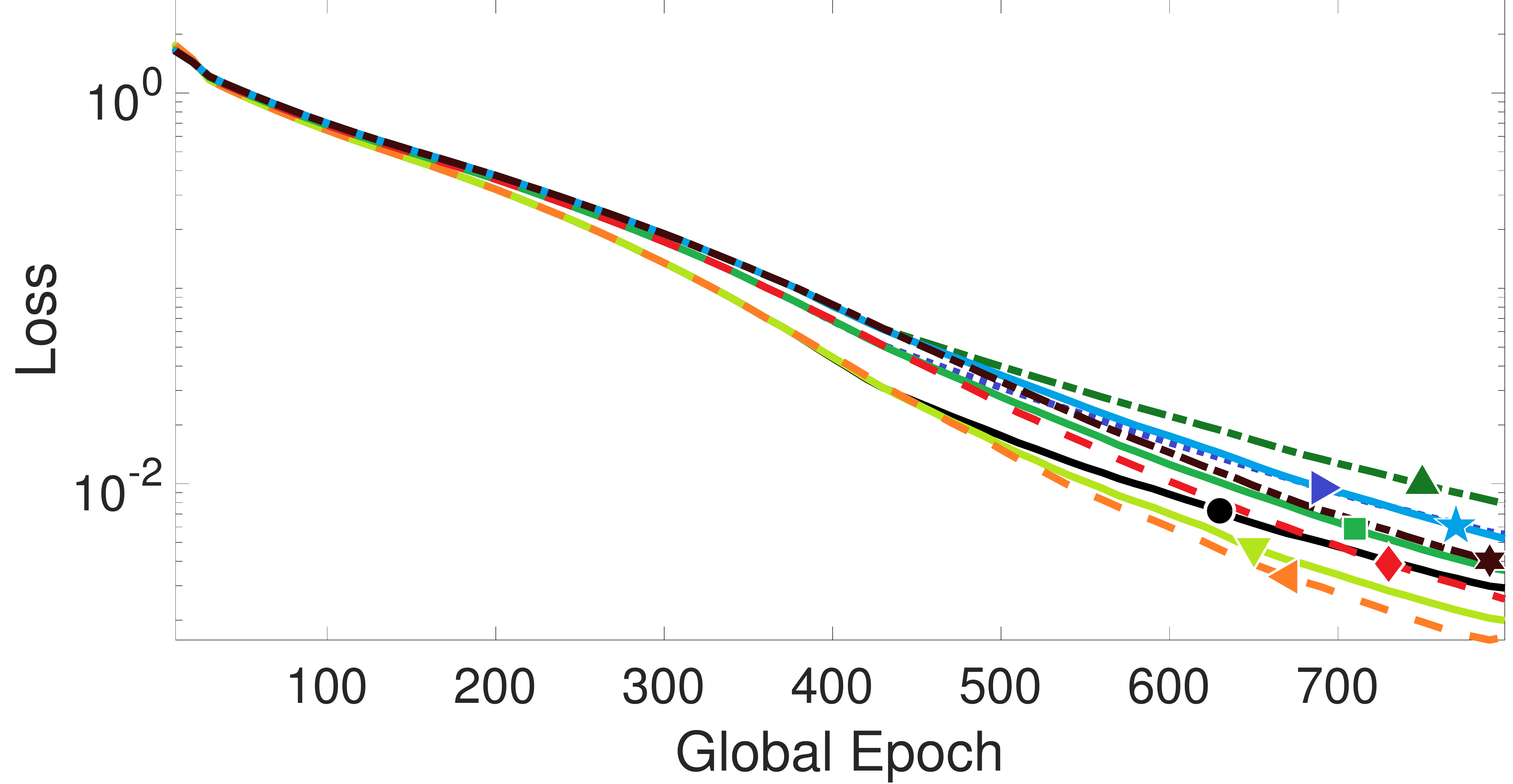}}
\caption{Convergence on training data with balanced partition, without attack. Each epoch is a full pass of the local training data. Legend ``SLSGD, $\gamma=0.1, \alpha=0.8, b=2$'' means that SLSGD takes the learning rate $0.1$ and $\trmean_2$ for aggregation, and the initial $\alpha=1$ decays by the factor of $0.8$ at the $400$th epoch. SLSGD with $\alpha=1$ and $b=0$ is the baseline \textit{FedAvg}. Note that we fix the random seeds. Thus, before $\alpha$ decays at the $400$th epoch, results with the same $\gamma$ and $b$ are the same.}
\label{fig:balanced}
\end{figure*}
\begin{figure*}[htb!]
\centering
\subfigure[Top-1 accuracy on testing set]{\includegraphics[width=0.49\textwidth,height=3.8cm]{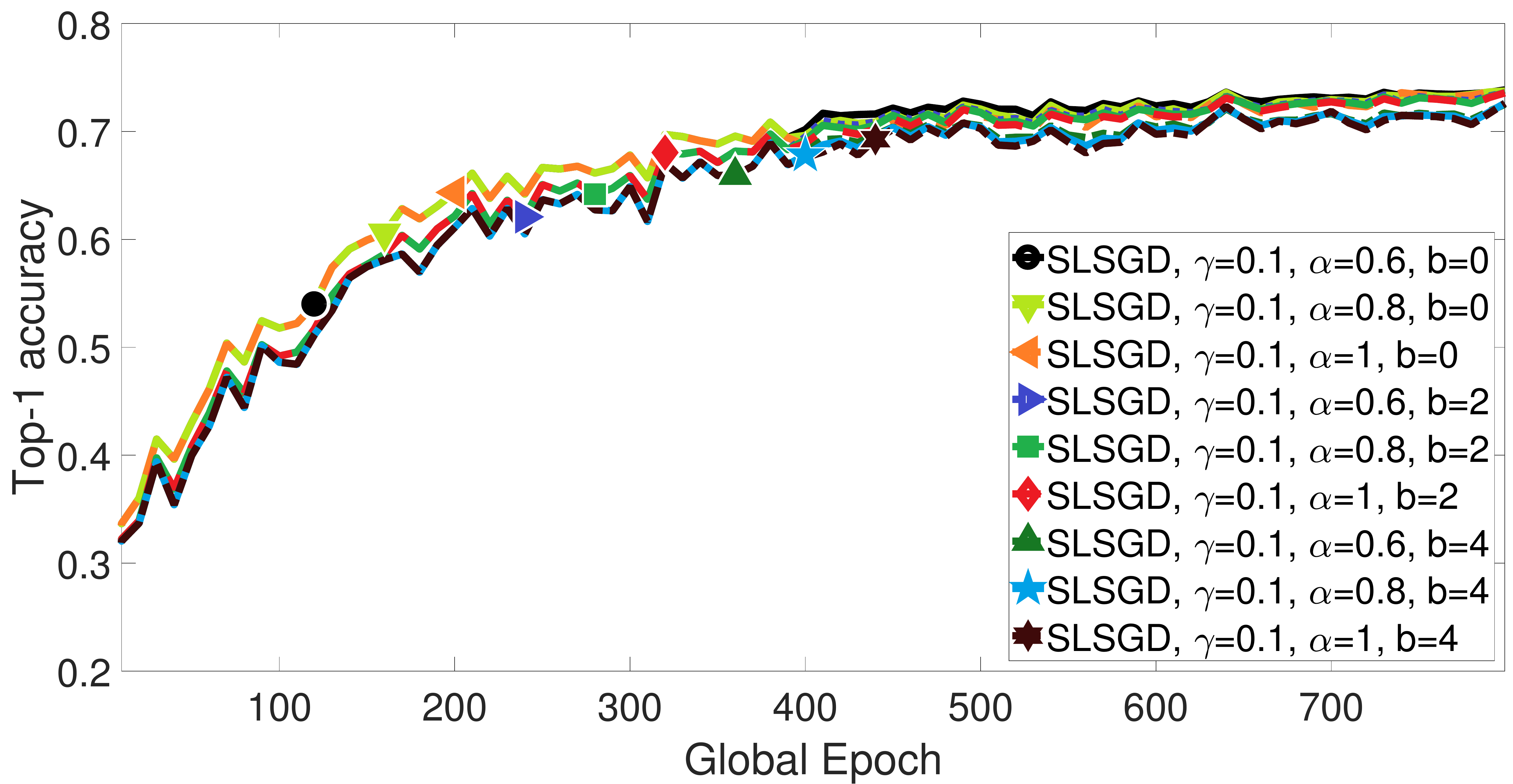}}
\subfigure[Cross entropy on training set]{\includegraphics[width=0.49\textwidth,height=3.8cm]{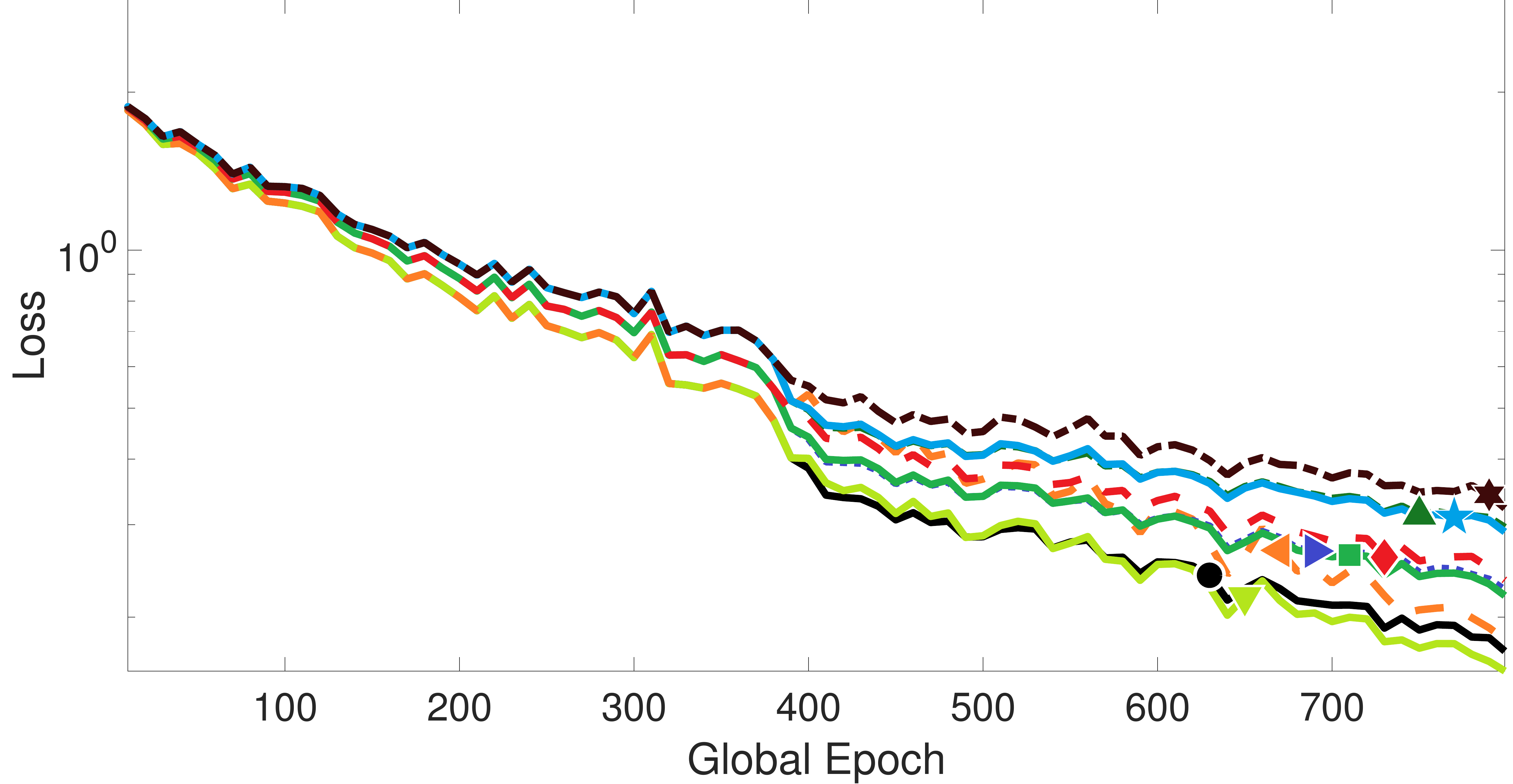}}
\caption{Convergence on training data with unbalanced partition, without attack. Each epoch is a full pass of the local training data. Legend ``SLSGD, $\gamma=0.1, \alpha=0.8, b=2$'' means that SLSGD takes the learning rate $0.1$ and $\trmean_2$ for aggregation, and the initial $\alpha=1$ decays by the factor of $0.8$ at the $400$th epoch. SLSGD with $\alpha=1$ and $b=0$ is the baseline \textit{FedAvg}. Note that we fix the random seeds. Thus, before $\alpha$ decays at the $400$th epoch, results with the same $\gamma$ and $b$ are the same.}
\label{fig:unbalanced}
\end{figure*}
\begin{figure*}[htb!]
\centering
\subfigure[Top-1 accuracy on testing set, $q=2$]{\includegraphics[width=0.49\textwidth,height=3.8cm]{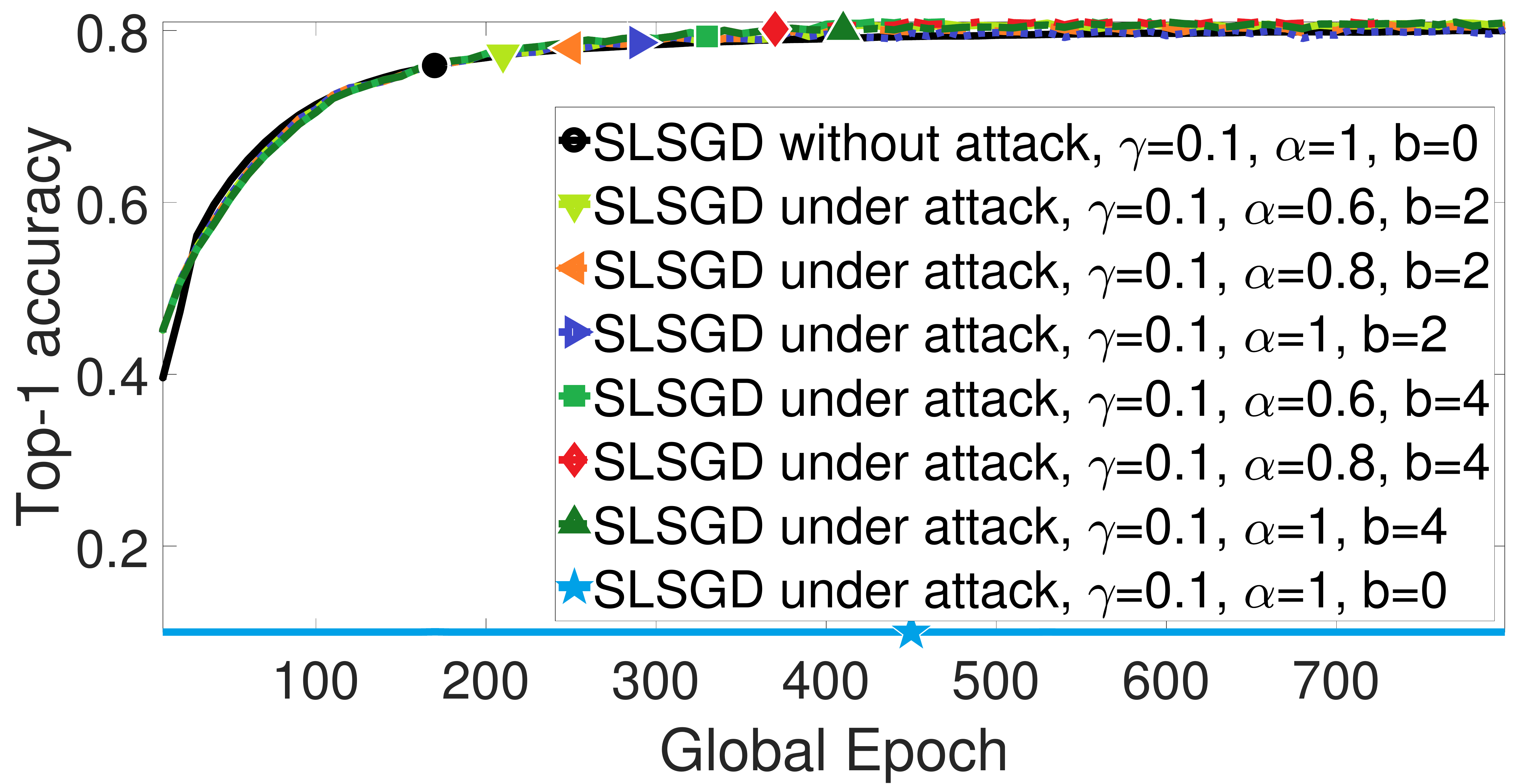}}
\subfigure[Cross entropy on training set, $q=2$]{\includegraphics[width=0.49\textwidth,height=3.8cm]{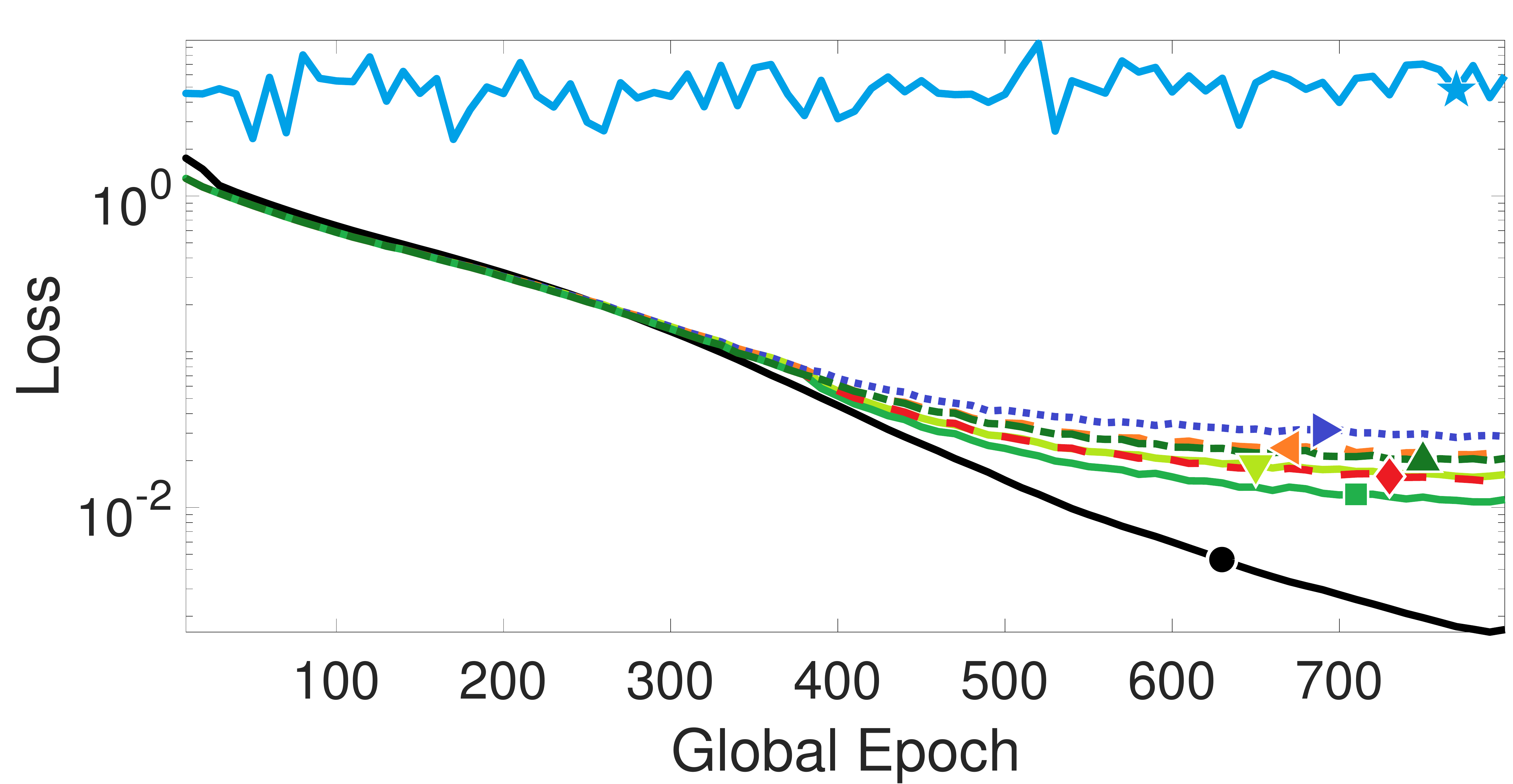}}
\subfigure[Top-1 accuracy on testing set, $q=4$]{\includegraphics[width=0.49\textwidth,height=3.8cm]{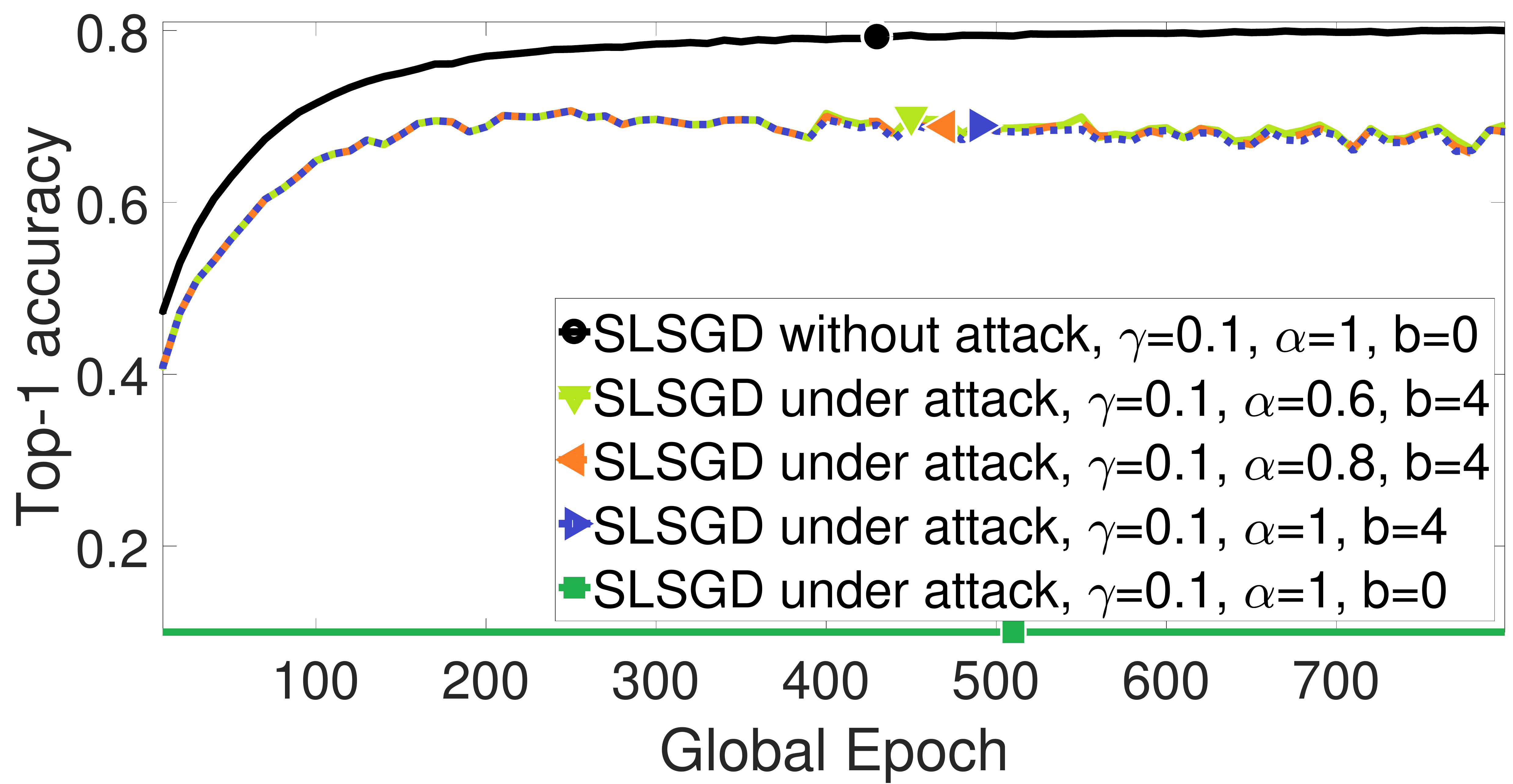}}
\subfigure[Cross entropy on training set, $q=4$]{\includegraphics[width=0.49\textwidth,height=3.8cm]{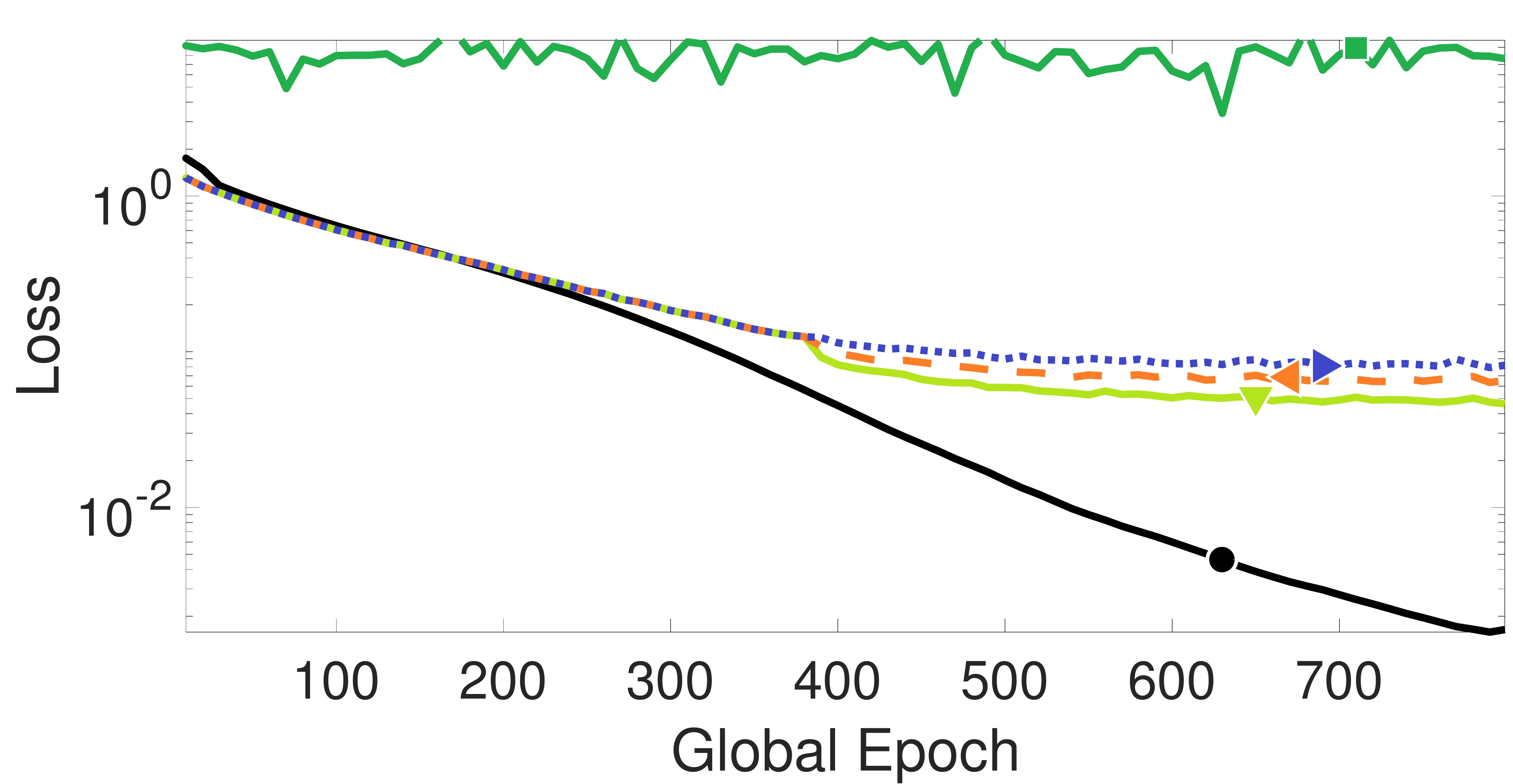}}
\caption{Convergence on training data with balanced partition, with ``label-flipping'' attack. In each epoch, we guarantee that $q \in \{2, 4\}$ of the $k=10$ selected workers are poisoned.  Each epoch is a full pass of the local training data. Legend ``SLSGD, $\gamma=0.1, \alpha=0.8, b=2$'' means that SLSGD takes the learning rate $0.1$ and $\trmean_2$ for aggregation, and the initial $\alpha=1$ decays by the factor of $0.8$ at the $400$th epoch. SLSGD with $\alpha=1$ and $b=0$ is the baseline \textit{FedAvg}. Note that we fix the random seeds. Thus, before $\alpha$ decays at the $400$th epoch, results with the same $\gamma$ and $b$ are the same.}
\label{fig:balanced_trim_labelflip}
\end{figure*}
\begin{figure*}[htb!]
\centering
\subfigure[Top-1 accuracy on testing set, $q=2$]{\includegraphics[width=0.49\textwidth,height=3.8cm]{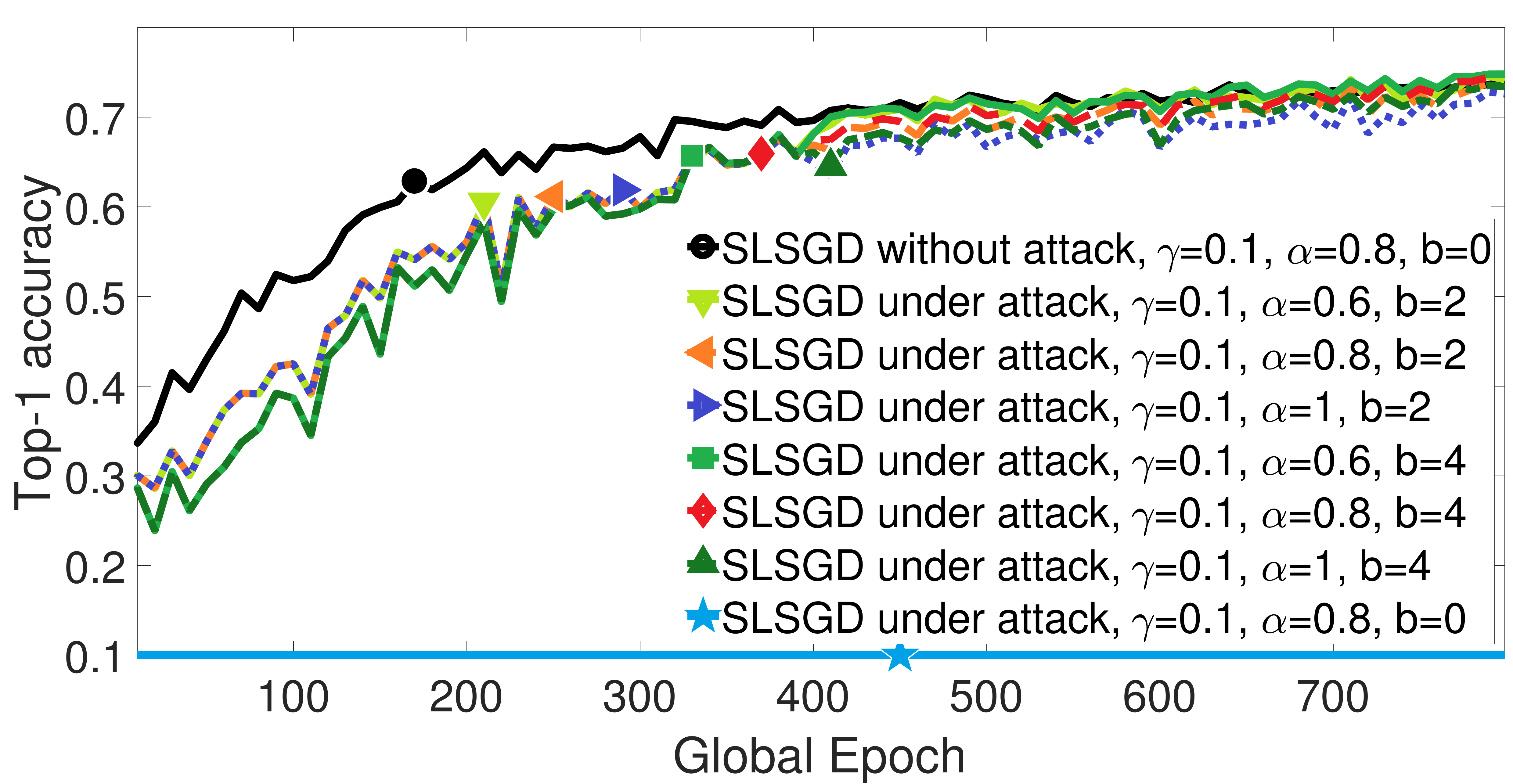}}
\subfigure[Cross entropy on training set, $q=2$]{\includegraphics[width=0.49\textwidth,height=3.8cm]{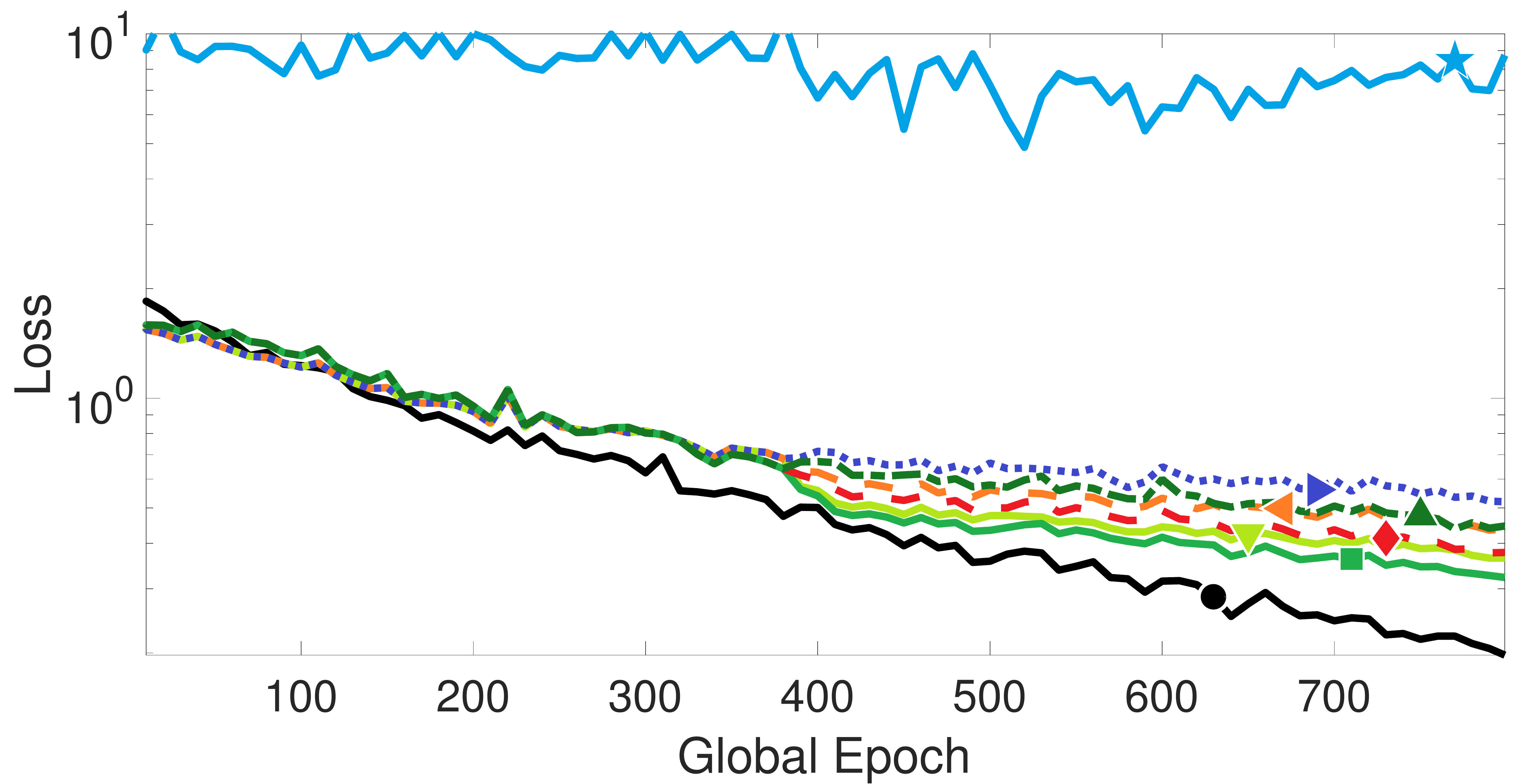}}
\subfigure[Top-1 accuracy on testing set, $q=4$]{\includegraphics[width=0.49\textwidth,height=3.8cm]{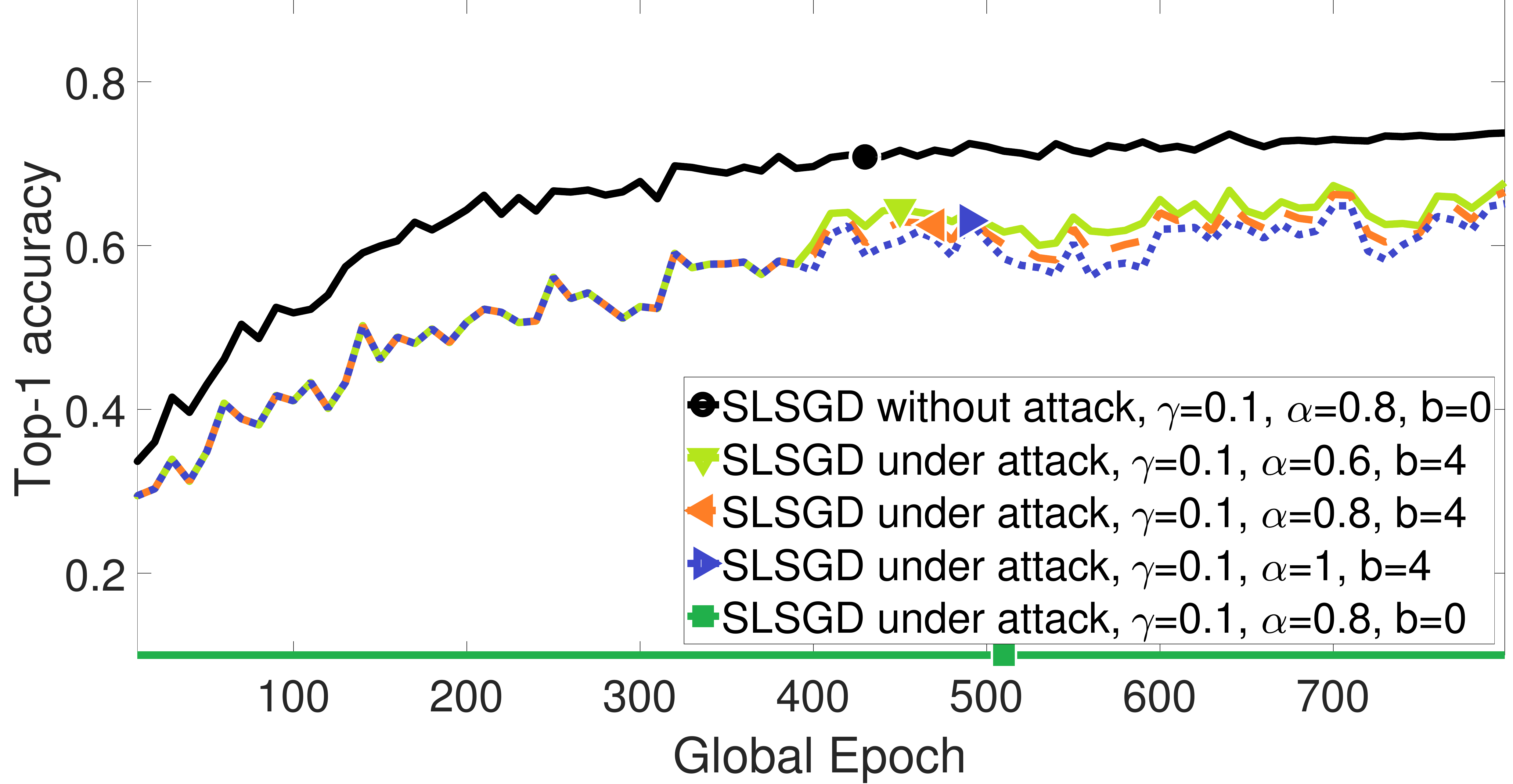}}
\subfigure[Cross entropy on training set, $q=4$]{\includegraphics[width=0.49\textwidth,height=3.8cm]{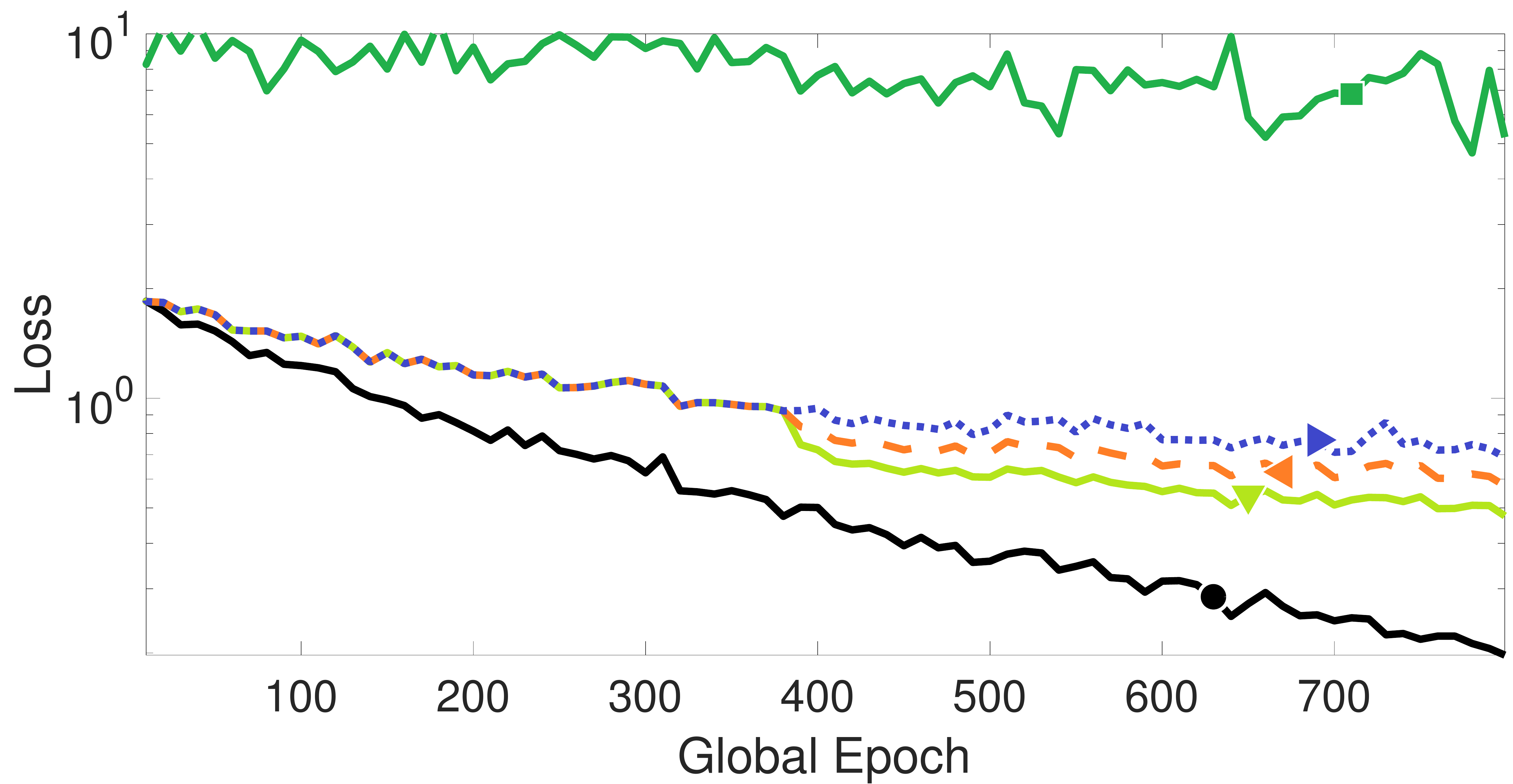}}
\caption{Convergence on training data with unbalanced partition, with ``label-flipping'' attack. In each epoch, we guarantee that $q$ of the $k=10$ selected workers are poisoned. Each epoch is a full pass of the local training data. Legend ``SLSGD, $\gamma=0.1, \alpha=0.8, b=2$'' means that SLSGD takes the learning rate $0.1$ and $\trmean_2$ for aggregation, and the initial $\alpha=1$ decays by the factor of $0.8$ at the $400$th epoch. SLSGD with $\alpha=1$ and $b=0$ is the baseline \textit{FedAvg}. Note that we fix the random seeds. Thus, before $\alpha$ decays at the $400$th epoch, results with the same $\gamma$ and $b$ are the same.}
\label{fig:unbalanced_trim_labelflip}
\end{figure*}

\subsection{SLSGD without Attack}

We first test the performance of SLSGD on the training data with balanced partition, without data poisoning. The result is shown in Fig.~\ref{fig:balanced}. When there are no poisoned workers, using trimmed mean results in extra variance. Although larger $b$ and smaller $\alpha$ makes the convergence slower, the gap is tiny. In general, SLSGD is insensitive to hyperparameters.

Then, we test the performance with unbalanced partition, without data poisoning. The result is shown in Fig.~\ref{fig:unbalanced}. Note that the convergence with unbalanced partition is generally slower compared to balanced partition due to the larger variance. Using appropriate $\alpha$~($\alpha=0.8$) can mitigate such extra variance. 

\subsection{SLSGD under Data Poisoning Attack}

To test the tolerance to poisoned workers, we simulate data poisoning which ``flips" the labels of the local training data. The poisoned data have ``flipped" labels, i.e., each $label \in \{ 0, \ldots, 9 \}$ in the local training data will be replaced by $(9 - label)$. The experiment is set up so that in each epoch, in all the $k=10$ randomly selected workers, $q$ workers are compromised and subjected to data poisoning. 
The results are shown in Fig.~\ref{fig:balanced_trim_labelflip} and Fig.~\ref{fig:unbalanced_trim_labelflip}. We use FedAvg/SLSGD without data poisoning~(Option I) as the ideal benchmark. As expected, SLSGD without trimmed mean can not tolerate data poisoning, which causes catastrophic failure. SLSGD with Option II tolerates the poisoned worker, though converges slower compared to SLSGD without data poisoning. Furthermore, larger $b$ and smaller $\alpha$ improves the robustness and stabilizes the convergence.

Note that taking $q=4$ in every epoch pushes to the limit of SLSGD since the algorithm requires $2q < k$. In practice, if there are totally $q=4$ poisoned workers in the entire $n=100$ workers, there is no guarantee that the poisoned workers will always be activated in each epoch. Poisoning 40\% of the sampled data in each epoch incurs huge noise, while SLSGD can still prevent the global model from divergence.

In Fig.~\ref{fig:alpha_b}, we show how $\alpha$ and $b$ affect the convergence when data poisoning and unbalanced partition cause extra error and variance. In such scenario, larger $b$ and smaller $\alpha$ makes SLSGD more robust and converge faster.

\begin{figure*}[htb!]
\centering
\includegraphics[width=0.75\textwidth,height=3cm]{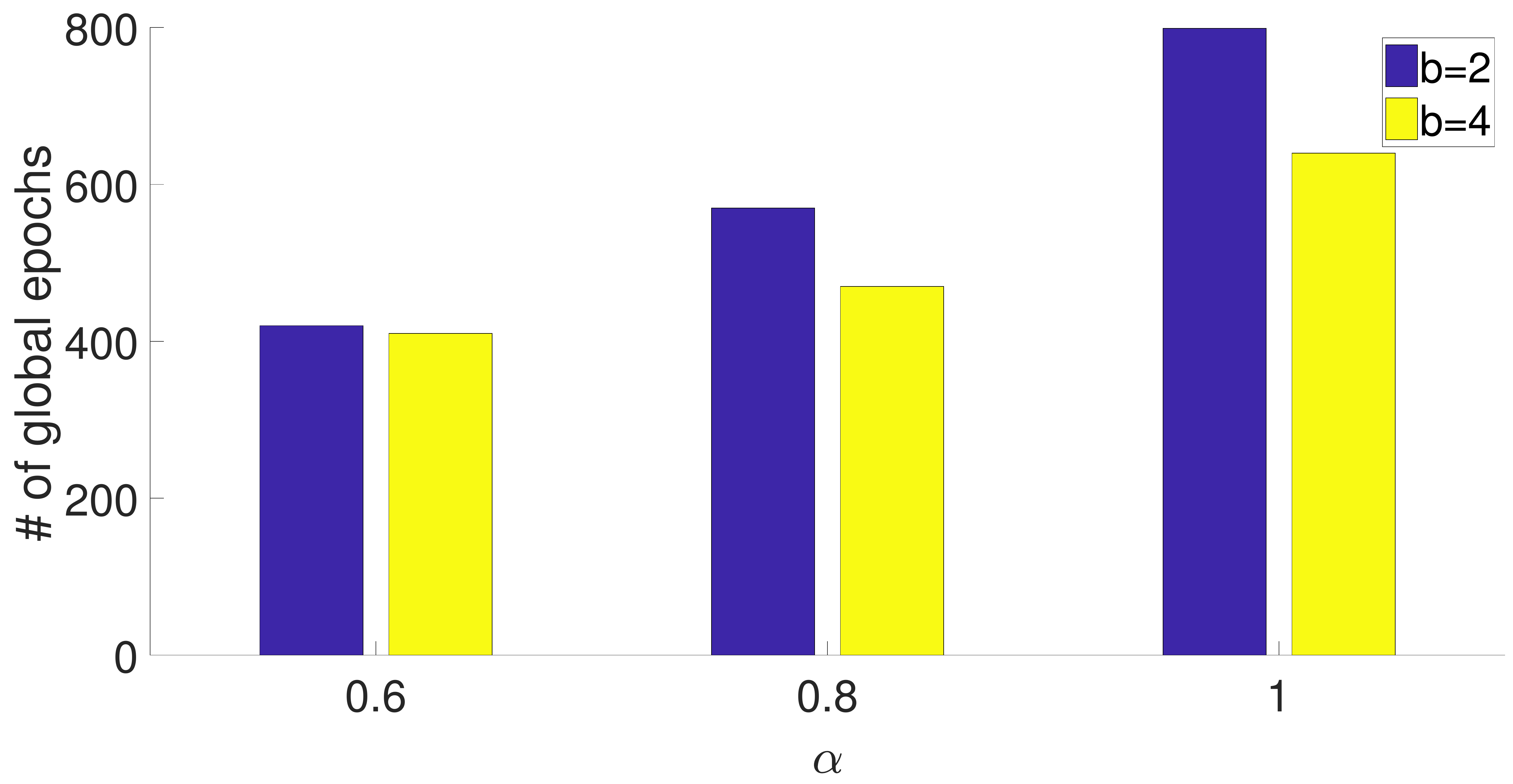}
\caption{Number of global epochs to reach training loss value $0.5$, with unbalanced partition and $q=2$ poisoned workers. $\gamma=0.1$. $\alpha$ and $b$ varies. ``$\alpha$'' on the x-axis is the initial value of $\alpha$, which does not decay during training.}
\label{fig:alpha_b}
\end{figure*}
\begin{figure*}[htb!]
\centering
\includegraphics[width=0.75\textwidth,height=3cm]{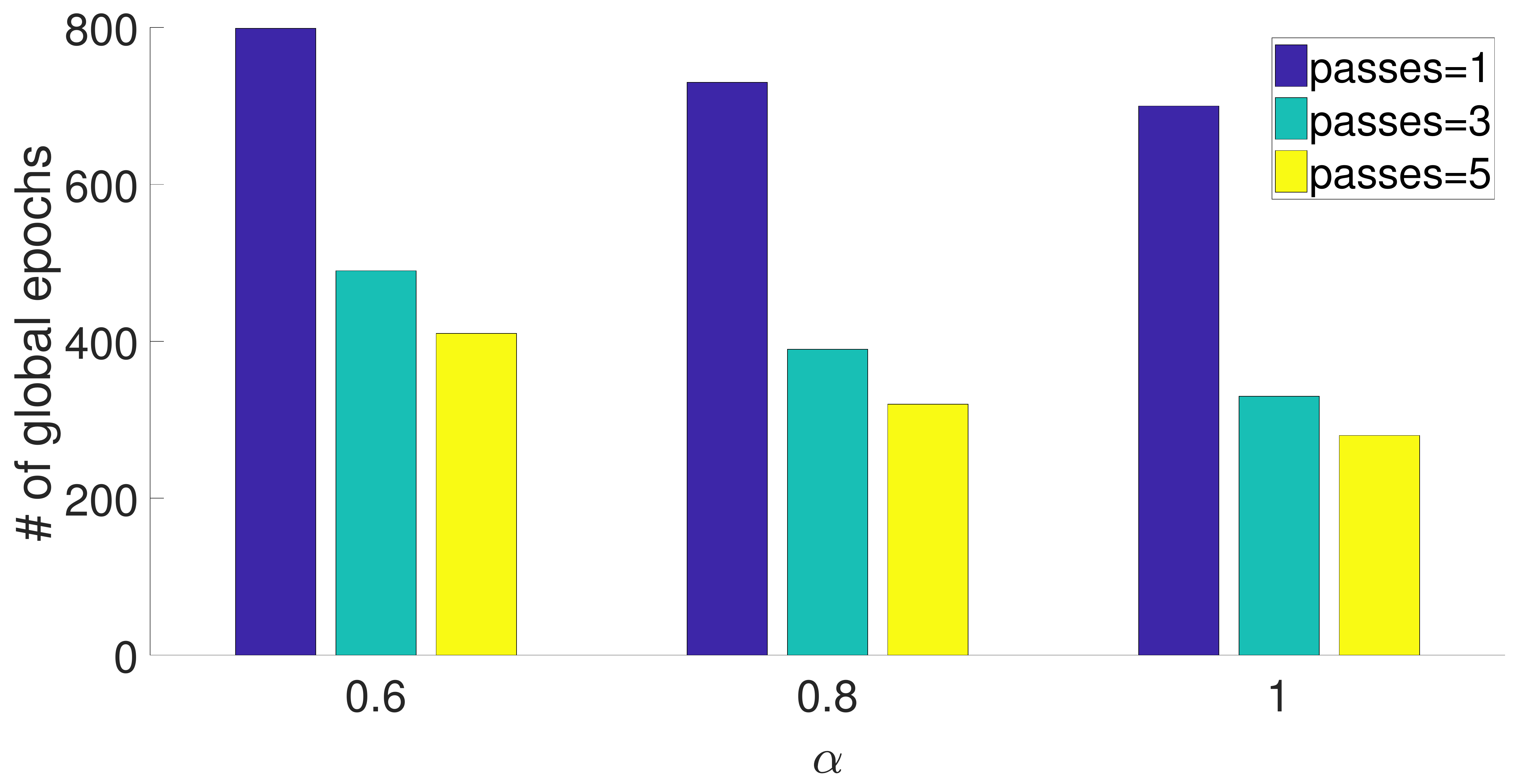}
\caption{Number of global epochs to reach training loss value $0.003$, with balanced partition, without poisoned workers. $\gamma=0.1$. $\alpha$ and number of local iterations varies. ``pass=3'' means each epoch is 3 full passes of the local datasets on the selected workers. ``$\alpha$'' on the x-axis is the initial value of $\alpha$, which does not decay during training.}
\label{fig:iterations}
\end{figure*}

\subsection{Acceleration by Local Updates}
According to our theoretical analysis, more local updates in each epoch accelerate the convergence. We test this theory in Fig.~\ref{fig:iterations} with balanced partition, without data poisoning. In the legend, ``pass=3'' means each epoch is 3 full passes of the local datasets~($H = 3 \times 500 / 50 = 30$ local iterations) on the selected workers. We show that with more local iterations, SLSGD converges faster.

\subsection{Discussion}

The hyperparameters of SLSGD affects the convergence differently in different scenarios:
\setitemize[0]{leftmargin=*}
\begin{itemize}
\item \textit{Balanced partition, no attacks.} In this case, the overall variance is relatively small. Thus, it is not necessary to use smaller $\alpha$ to mitigate the variance. The extra variance caused by trimmed mean slows down the convergence.
Since the variance does not dominate, smaller $\alpha$ and larger $b$ potentially slow down the convergence, but the gap is tiny.

\item \textit{Unbalanced partition, no attacks.} In this case, the overall variance is larger than the balanced case. 
Note that not only the size of local datasets, but also the label distribution are unbalanced among the devices. Some partitions only contains one label, which enlarges the accumulative error caused by infrequent synchronization and overfitting the local training data.
Thus, using appropriate $\alpha$ can mitigate the variance. However, it is not necessary to use the trimmed mean, since the variance caused by unbalanced partition is not too bad compared to data poisoning.

\item \textit{Balanced partition, under attacks.} In this case, the error caused by poisoned workers dominates. We must use trimmed mean to prevent divergence.  Larger $b$ improves the robustness and convergence. Furthermore, using smaller $\alpha$ also mitigates the error and improves the convergence.

\item \textit{Unbalanced partition, under attacks.} In this case, the error caused by poisoned workers still dominates. In general, the usage of hyperparameters is similar to the case of balanced partition under attacks. However, the unbalanced partition makes it more difficult to distinguish poisoned workers from normal workers. As a result, the convergence gets much slower. Smaller $\alpha$ obtain more improvement and better stabilization.
\end{itemize}

In general, there is a trade-off between convergence rate and variance/error reduction. In the ideal case, if the variance is very small, SLSGD with $\alpha=1$ and $b=0$, i.e., \textit{FedAvg}, has fastest convergence. Using other hyperparameters slightly slows down the convergence, but the gap is tiny. When variance gets larger, users can try smaller $\alpha$. When the variance/error gets catastrophically large, the users can use the trimmed mean to prevent divergence.

\section{Conclusion}

We propose a novel distributed optimization algorithm on non-IID training data, which has limited communication and tolerates poisoned workers. The algorithm has provable convergence. Our empirical results show good performance in practice. In future work, we are going to analyze our algorithm on other threat models, such as hardware or software failures.

\bibliography{fedrob_ecml2019}
\bibliographystyle{splncs04}

\newpage
\onecolumn

\vspace*{0.1cm}
\begin{center}
	\Large\textbf{Appendix}
\end{center}
\vspace*{0.1cm}

\section{Additional Experiments}

\subsection{CNN on CIFAR-10}

We conduct experiments on the benchmark CIFAR-10 image classification dataset~\cite{krizhevsky2009learning}, which is composed of 50k images for training and 10k images for testing. Each image is resized and cropped to the shape of $(24,24,3)$. We use convolutional neural network~(CNN) with 4 convolutional layers followed by 1 fully connected layer. We use a simple network architecture, so that it can be easily handled by edge devices.  The detailed network architecture can be found in our submitted source code (will also be released upon publication).
The experiments are conducted on CPU devices. We implement SLSGD using the MXNET~\cite{chen2015mxnet} framework. The results are shown in Fig. \ref{fig:balanced_appendix}, \ref{fig:unbalanced_appendix}, \ref{fig:balanced_trim_labelflip_appendix}, and \ref{fig:unbalanced_trim_labelflip_appendix}.

\begin{figure*}[htbp!]
\centering
\subfigure[Top-1 accuracy on testing set, $\gamma=0.1$]{\includegraphics[width=0.8\textwidth,height=4cm]{fedrob_acc_balanced_nobyz_1}}
\subfigure[Cross entropy on training set, $\gamma=0.1$]{\includegraphics[width=0.8\textwidth,height=4cm]{fedrob_loss_balanced_nobyz_1}}
\subfigure[Top-1 accuracy on testing set, $\gamma=0.05$]{\includegraphics[width=0.8\textwidth,height=4cm]{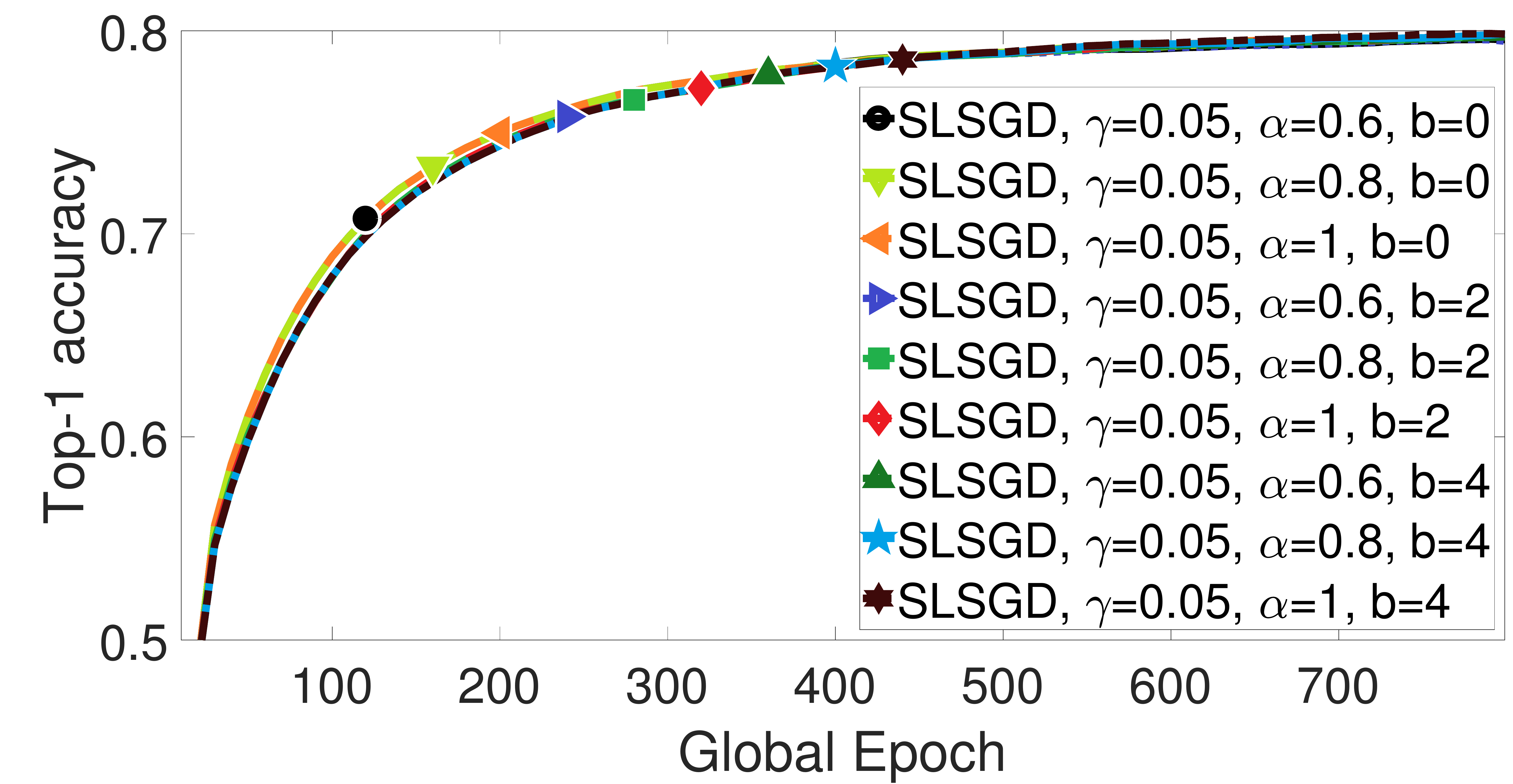}}
\subfigure[Cross entropy on training set, $\gamma=0.05$]{\includegraphics[width=0.8\textwidth,height=4cm]{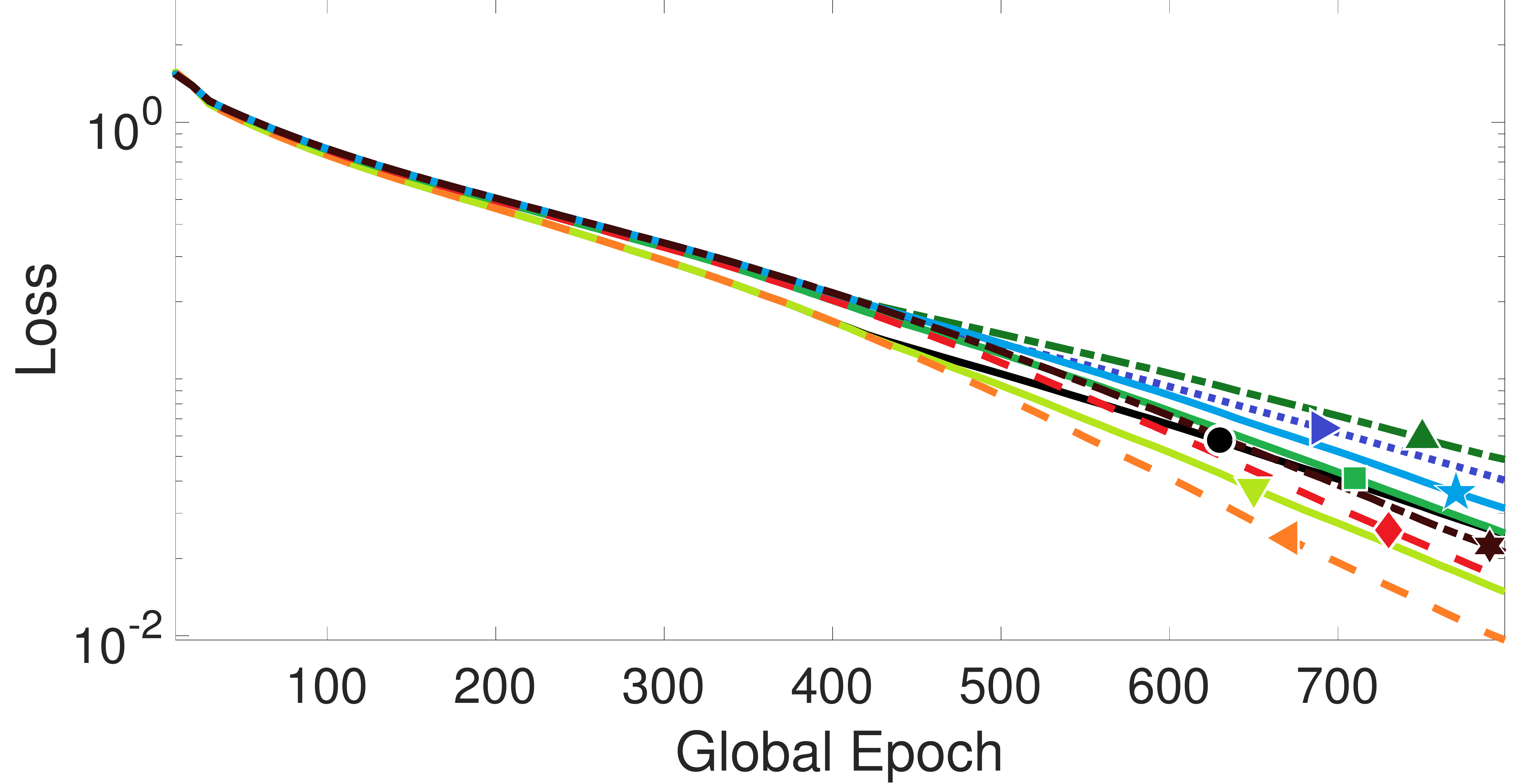}}
\caption{CNN experiment on CIFAR-10. Convergence on training data with unbalanced partition, without attack. Each epoch is a full pass of the local training data. Legend ``SLSGD, $\gamma=0.1, \alpha=0.8, b=2$'' means that SLSGD takes the learning rate $0.1$ and $\trmean_2$ for aggregation, and the initial $\alpha=1$ decays by the factor of $0.8$ at the $400$th epoch. Note that SLSGD with $\alpha=1$ and $b=0$ is the baseline \textit{FedAvg}. Note that we fix the random seeds. Thus, before $\alpha$ decays at the $400$th epoch, results with the same $\gamma$ and $b$ are the same.}
\label{fig:balanced_appendix}
\end{figure*}

\begin{figure*}[htbp!]
\centering
\subfigure[Top-1 accuracy on testing set, $\gamma=0.1$]{\includegraphics[width=0.8\textwidth,height=4cm]{fedrob_acc_unbalanced_nobyz_1}}
\subfigure[Cross entropy on training set, $\gamma=0.1$]{\includegraphics[width=0.8\textwidth,height=4cm]{fedrob_loss_unbalanced_nobyz_1}}
\subfigure[Top-1 accuracy on testing set, $\gamma=0.05$]{\includegraphics[width=0.8\textwidth,height=4cm]{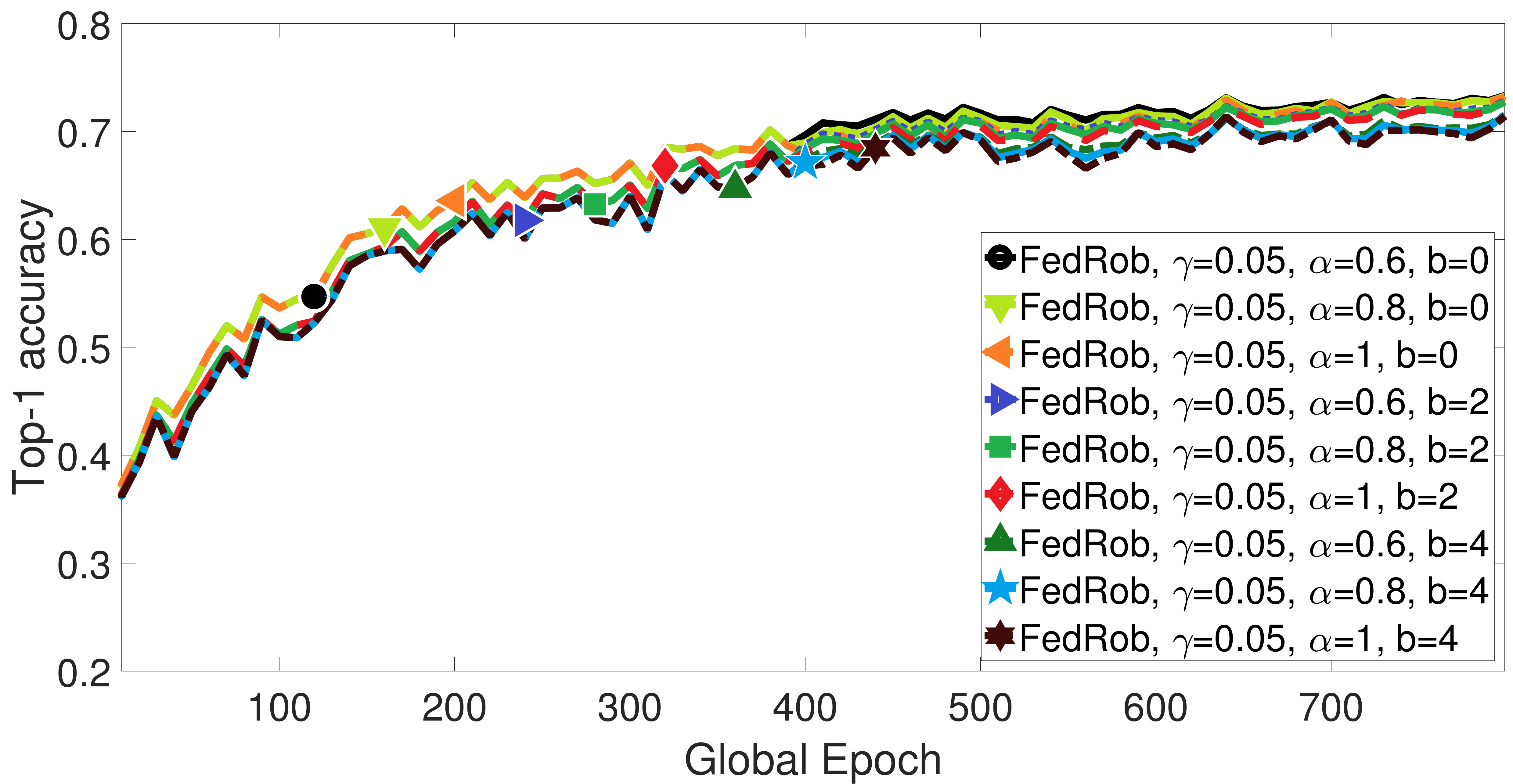}}
\subfigure[Cross entropy on training set, $\gamma=0.05$]{\includegraphics[width=0.8\textwidth,height=4cm]{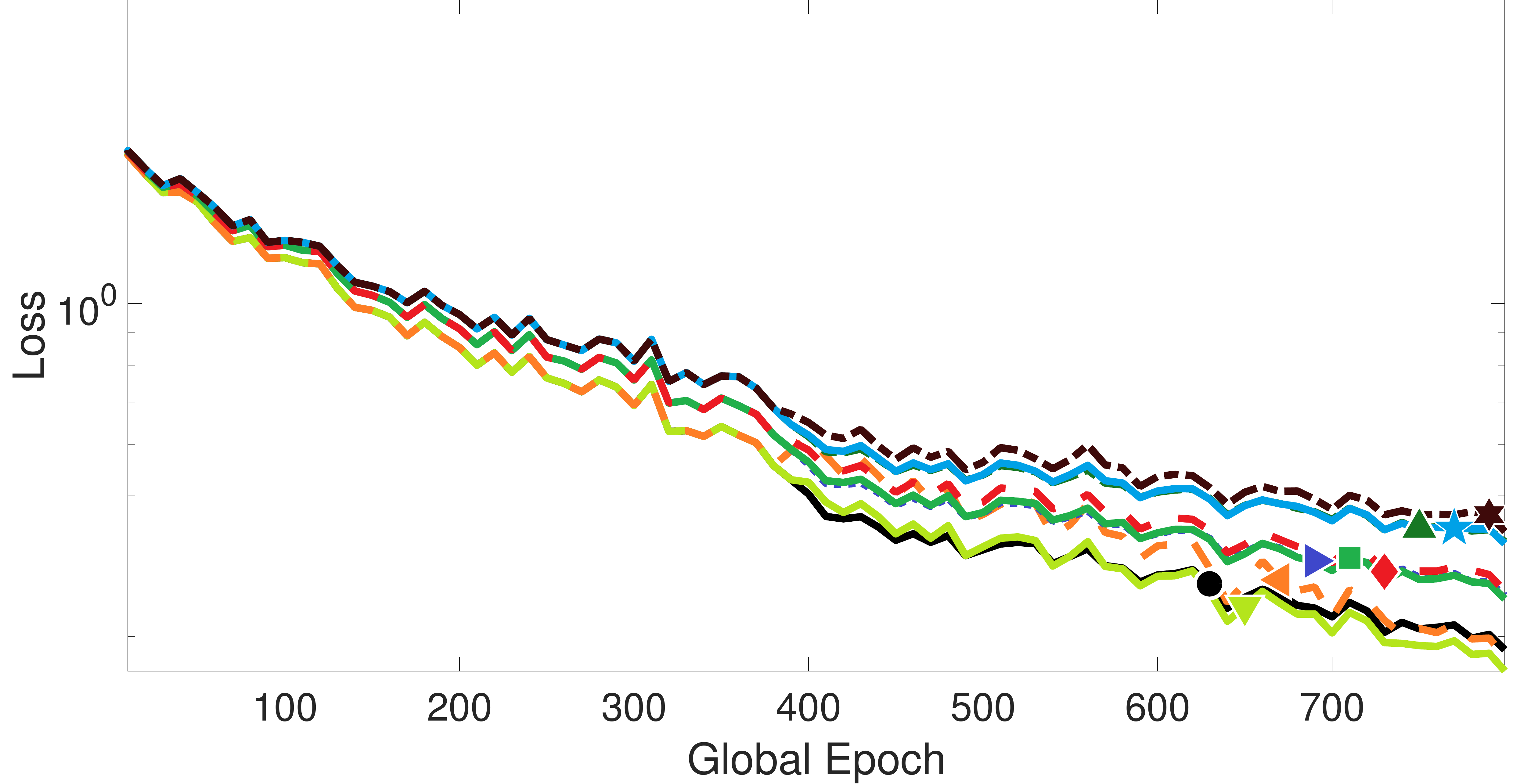}}
\caption{CNN experiment on CIFAR-10. Convergence on training data with unbalanced partition, without attack. Each epoch is a full pass of the local training data. Legend ``SLSGD, $\gamma=0.1, \alpha=0.8, b=2$'' means that SLSGD takes the learning rate $0.1$ and $\trmean_2$ for aggregation, and the initial $\alpha=1$ decays by the factor of $0.8$ at the $400$th epoch. Note that SLSGD with $\alpha=1$ and $b=0$ is the baseline \textit{FedAvg}. Note that we fix the random seeds. Thus, before $\alpha$ decays at the $400$th epoch, results with the same $\gamma$ and $b$ are the same.}
\label{fig:unbalanced_appendix}
\end{figure*}

\begin{figure*}[htbp!]
\centering
\subfigure[Top-1 accuracy on testing set, $q=2$]{\includegraphics[width=0.8\textwidth,height=4cm]{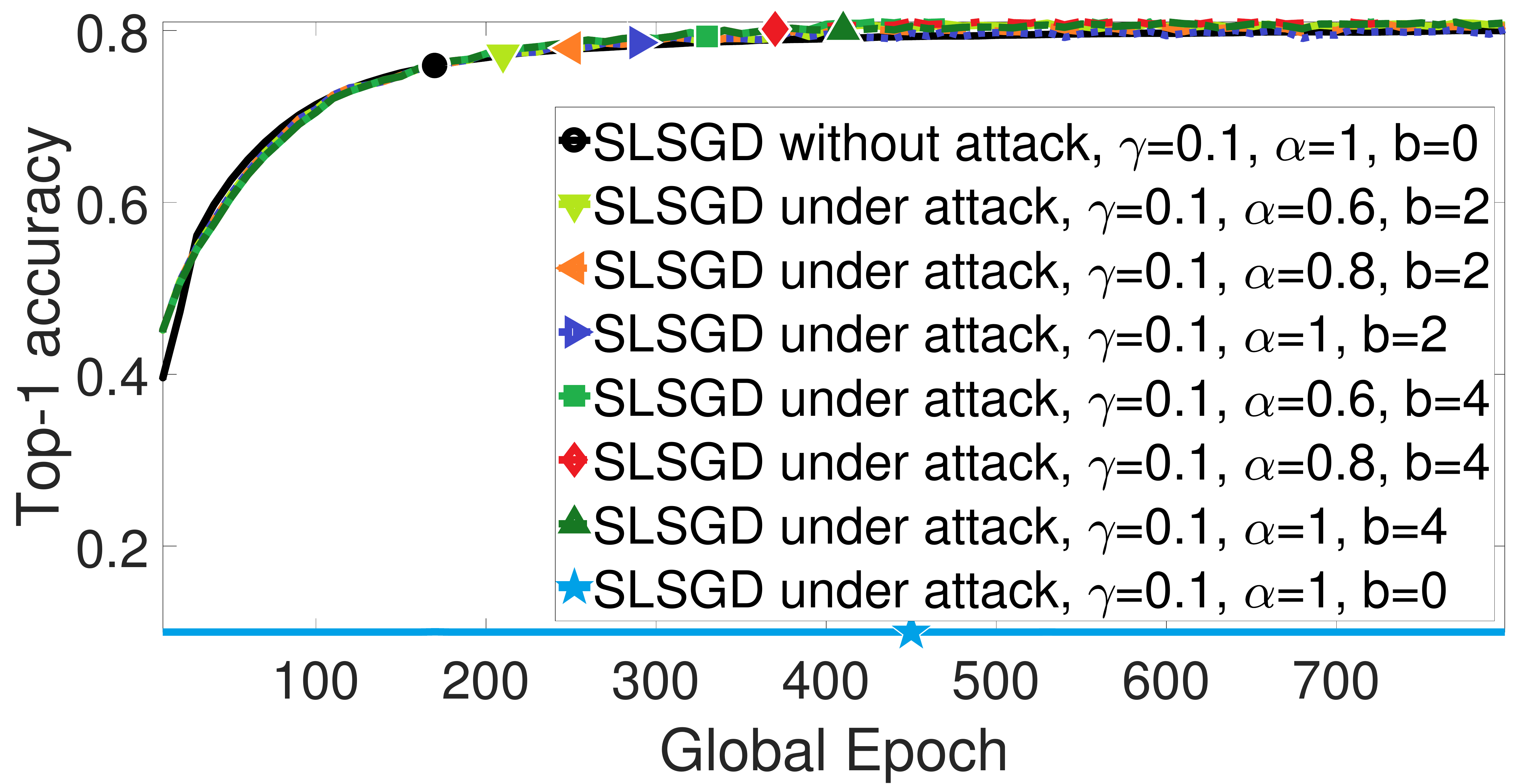}}
\subfigure[Cross entropy on training set, $q=2$]{\includegraphics[width=0.8\textwidth,height=4cm]{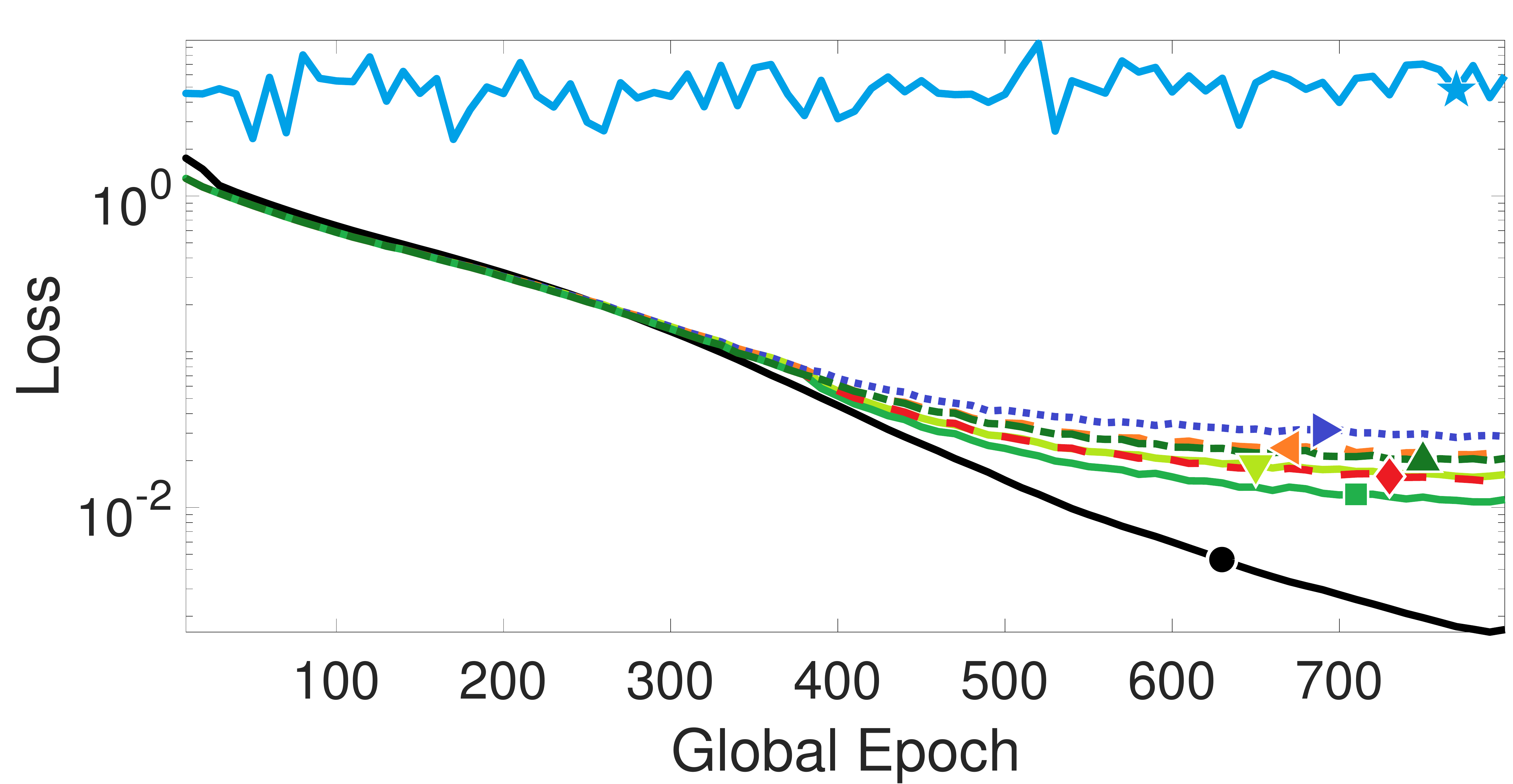}}
\subfigure[Top-1 accuracy on testing set, $q=4$]{\includegraphics[width=0.8\textwidth,height=4cm]{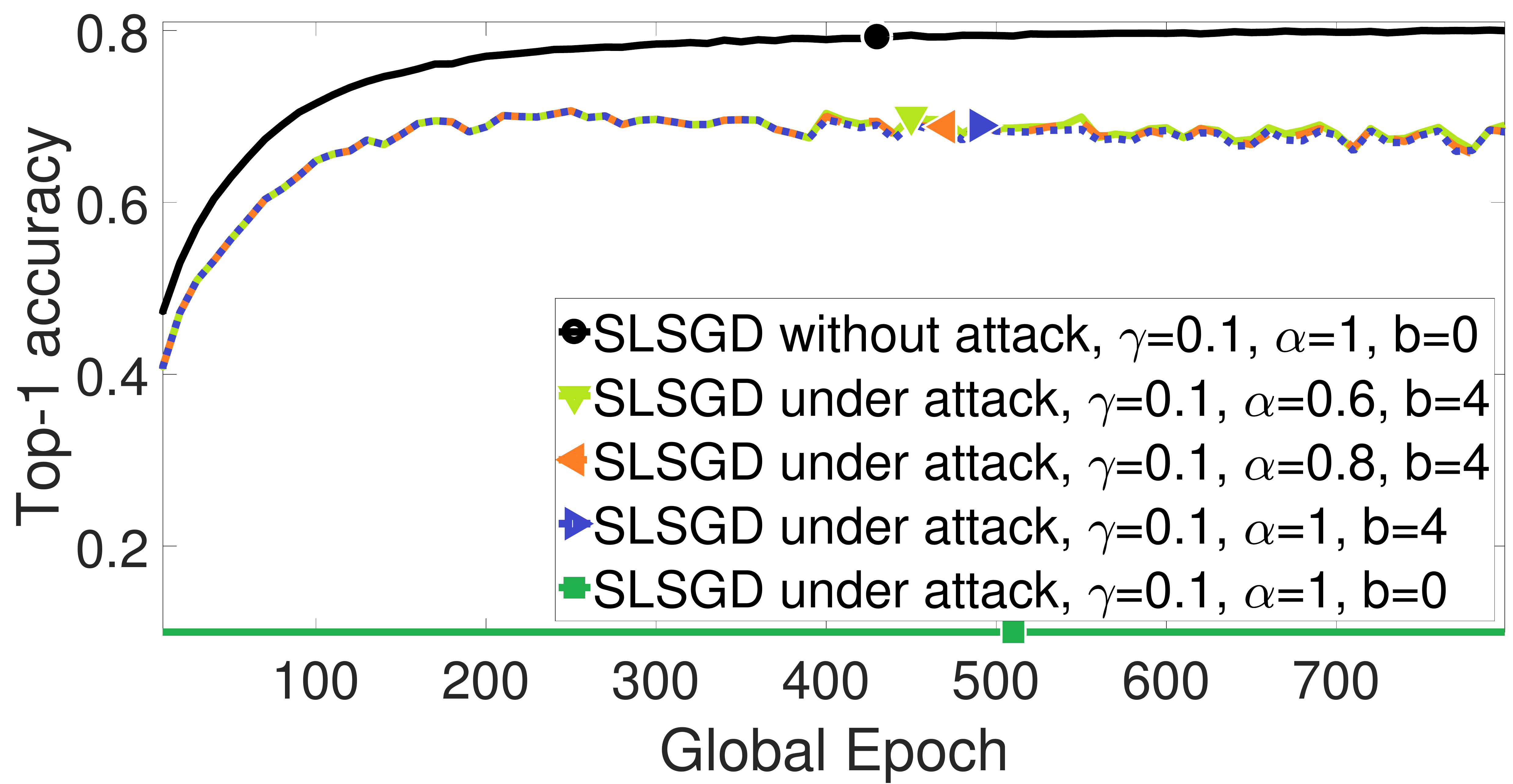}}
\subfigure[Cross entropy on training set, $q=4$]{\includegraphics[width=0.8\textwidth,height=4cm]{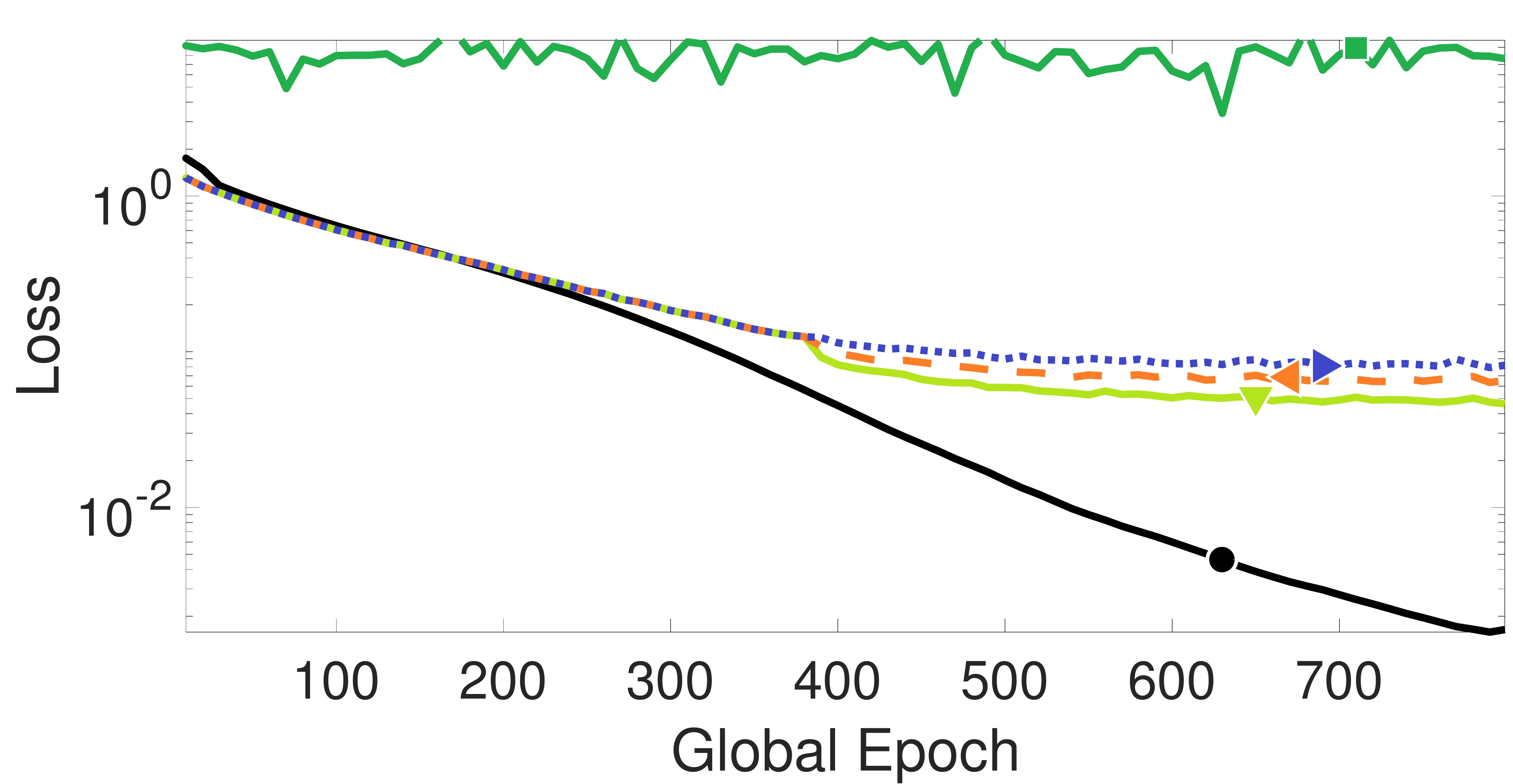}}
\caption{CNN experiment on CIFAR-10. Convergence on training data with balanced partition, with ``label-flipping'' attack. In each epoch, we guarantee that $q \in \{2, 4\}$ of the $k=10$ selected workers are poisoned.  Each epoch is a full pass of the local training data. Legend ``SLSGD, $\gamma=0.1, \alpha=0.8, b=2$'' means that SLSGD takes the learning rate $0.1$ and $\trmean_2$ for aggregation, and the initial $\alpha=1$ decays by the factor of $0.8$ at the $400$th epoch. Note that SLSGD with $\alpha=1$ and $b=0$ is the baseline \textit{FedAvg}. Note that we fix the random seeds. Thus, before $\alpha$ decays at the $400$th epoch, results with the same $\gamma$ and $b$ are the same.}
\label{fig:balanced_trim_labelflip_appendix}
\end{figure*}
\begin{figure*}[htbp!]
\centering
\subfigure[Top-1 accuracy on testing set, $q=2$]{\includegraphics[width=0.8\textwidth,height=4cm]{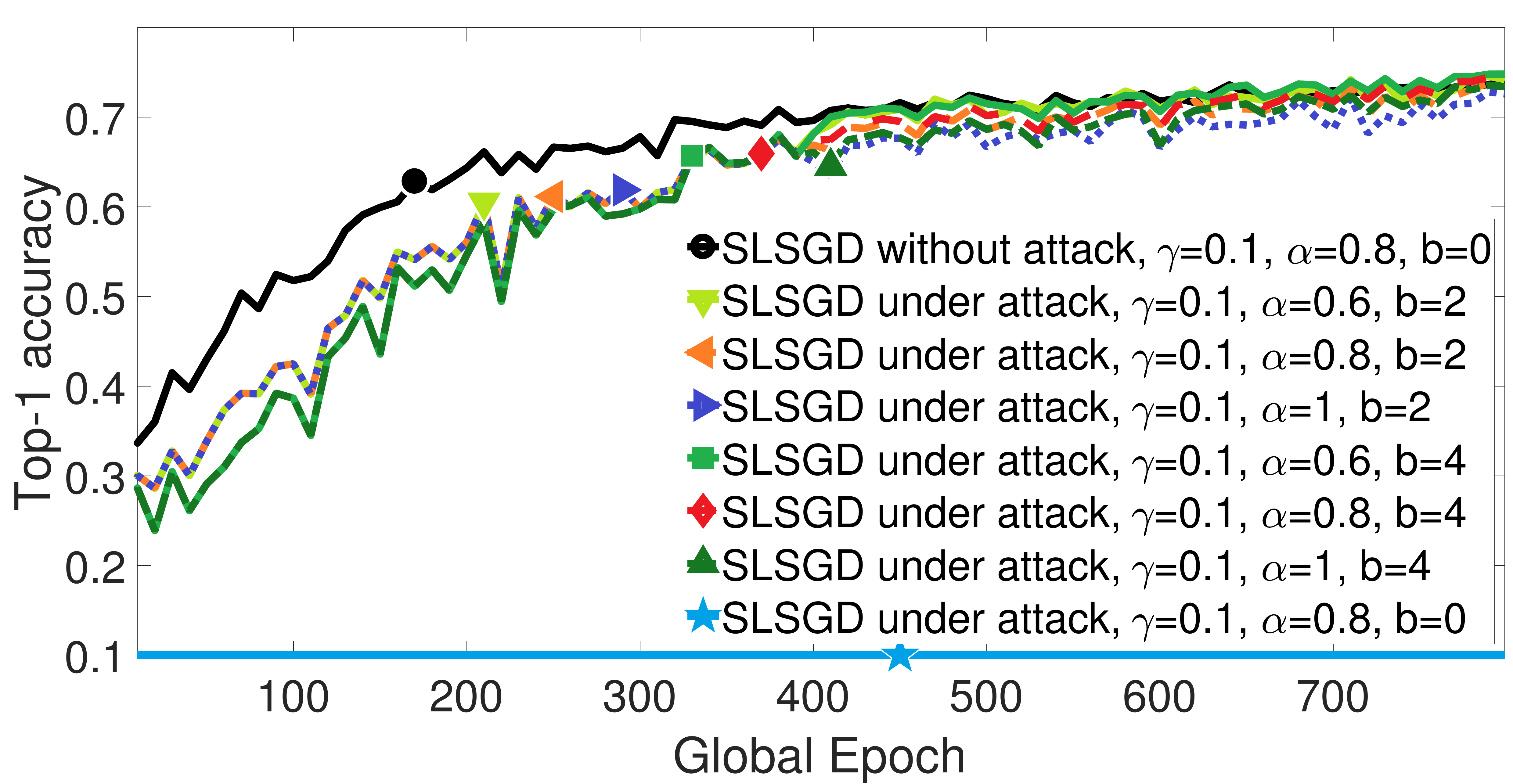}}
\subfigure[Cross entropy on training set, $q=2$]{\includegraphics[width=0.8\textwidth,height=4cm]{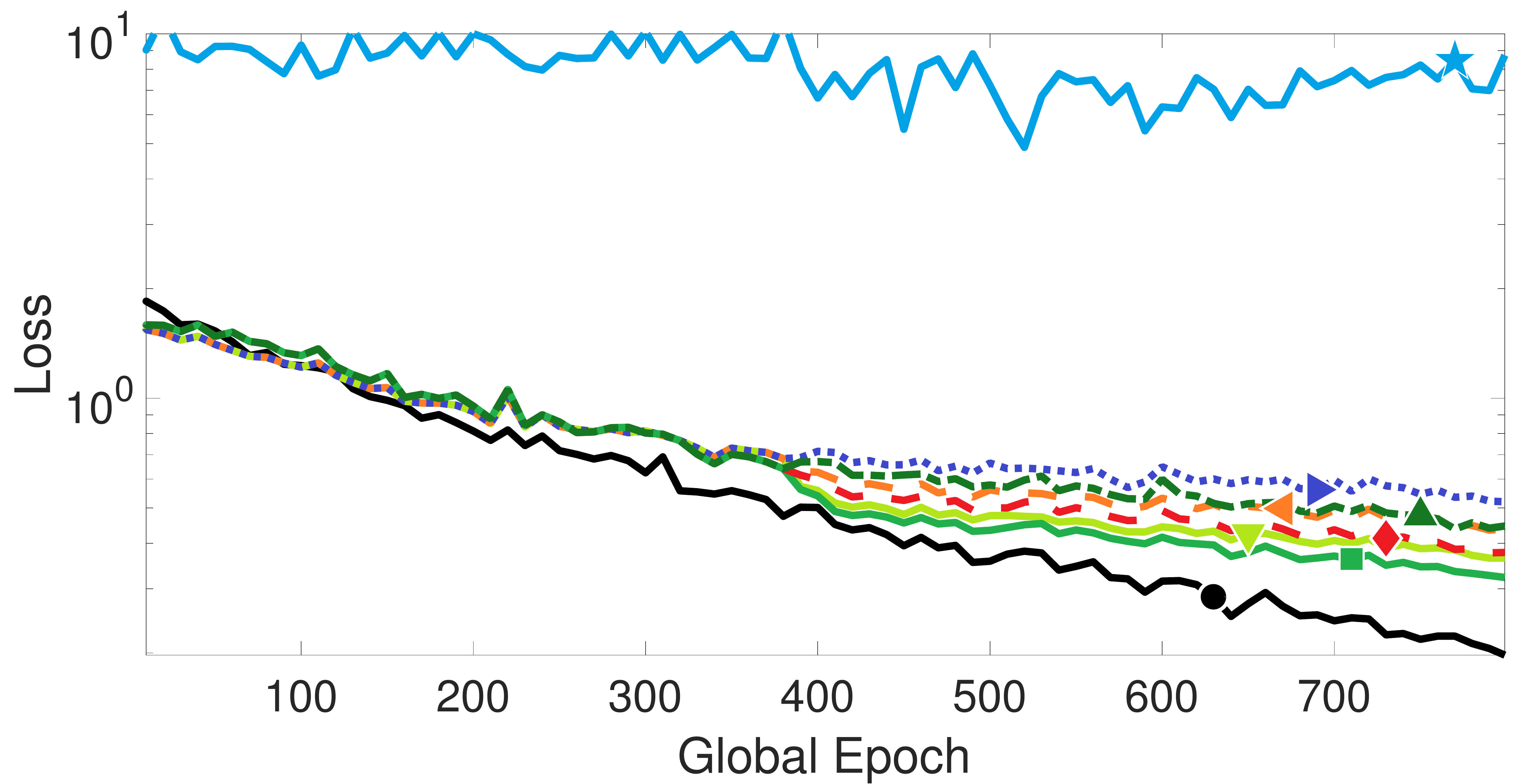}}
\subfigure[Top-1 accuracy on testing set, $q=4$]{\includegraphics[width=0.8\textwidth,height=4cm]{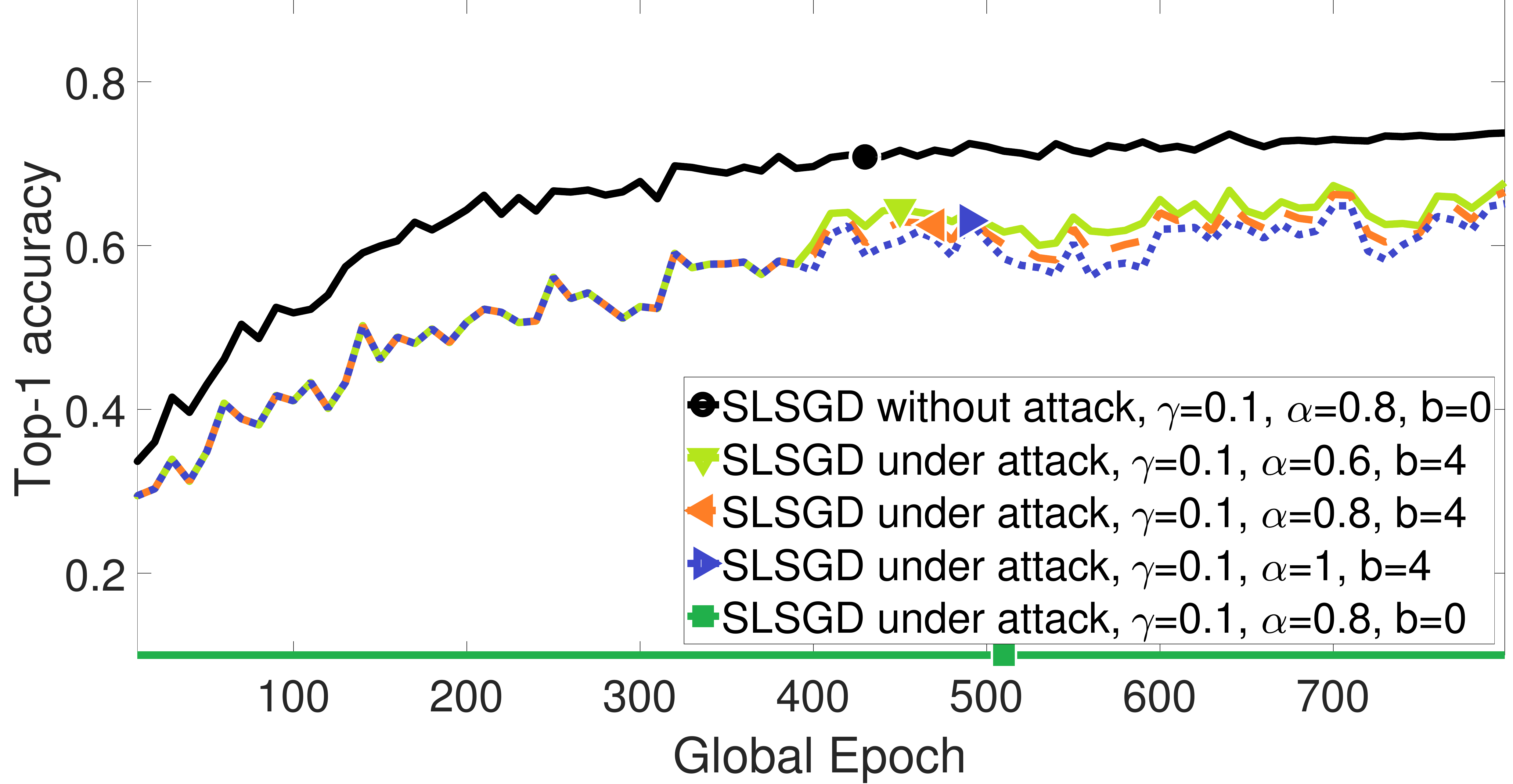}}
\subfigure[Cross entropy on training set, $q=4$]{\includegraphics[width=0.8\textwidth,height=4cm]{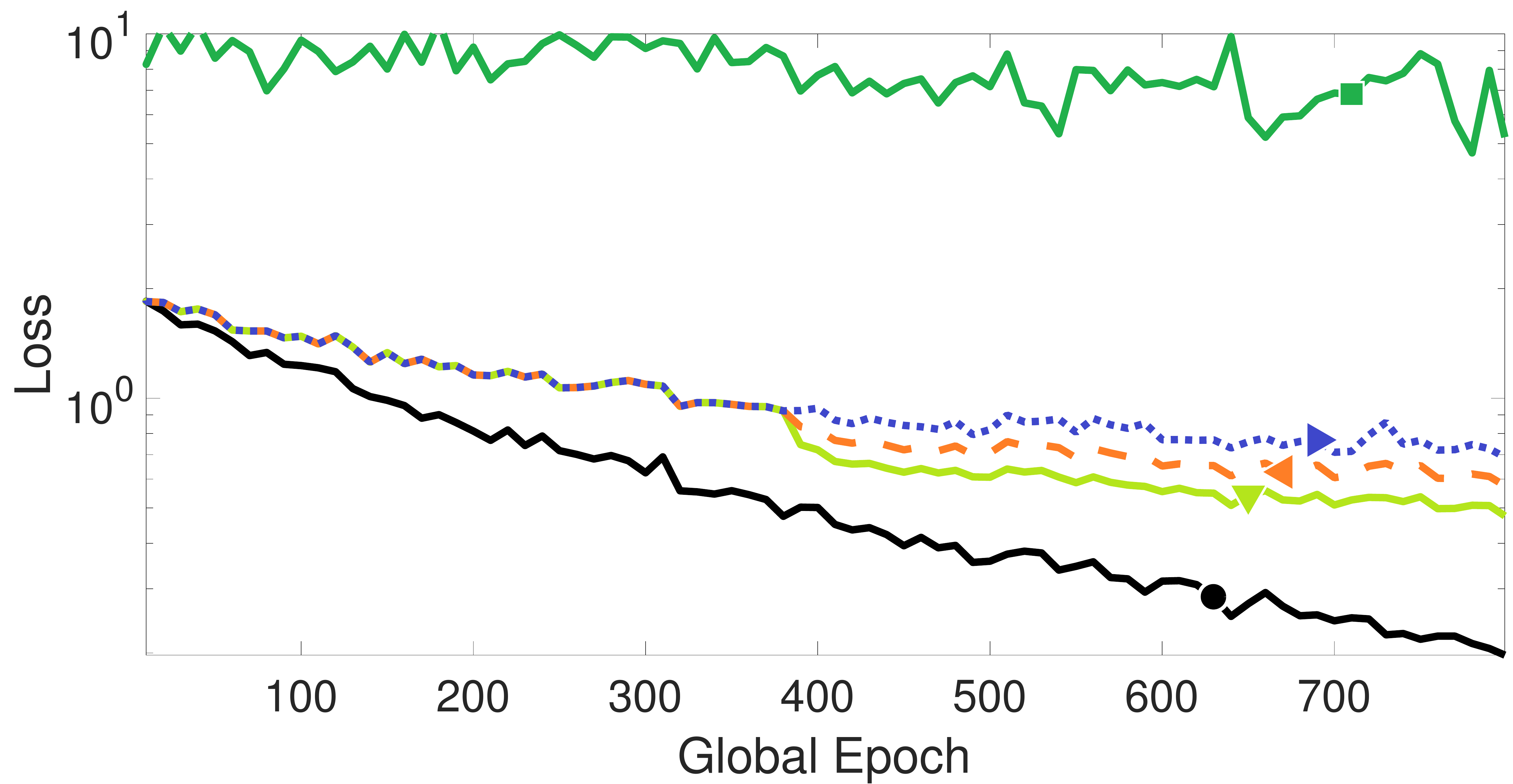}}
\caption{CNN experiment on CIFAR-10. Convergence on training data with unbalanced partition, with ``label-flipping'' attack. In each epoch, we guarantee that $q$ of the $k=10$ selected workers are poisoned. Each epoch is a full pass of the local training data. Legend ``SLSGD, $\gamma=0.1, \alpha=0.8, b=2$'' means that SLSGD takes the learning rate $0.1$ and $\trmean_2$ for aggregation, and the initial $\alpha=1$ decays by the factor of $0.8$ at the $400$th epoch. Note that SLSGD with $\alpha=1$ and $b=0$ is the baseline \textit{FedAvg}. Note that we fix the random seeds. Thus, before $\alpha$ decays at the $400$th epoch, results with the same $\gamma$ and $b$ are the same.}
\label{fig:unbalanced_trim_labelflip_appendix}
\end{figure*}

\subsection{LSTM on WikiText-2}

We conduct experiments on WikiText-2 dataset~\cite{merity2016pointer}. We use LSTM-based language model. The model architecture was taken from the MXNET and Gluon-NLP tutorial~\cite{gluonnlp}. The detailed network architecture can be found in our submitted source code (will also be released upon publication).
The experiments are conducted on CPU devices. We implement SLSGD using the MXNET~\cite{chen2015mxnet} framework.

Similar to the CNN experiments, the dataset is partitioned onto 100 devices. In each global epoch, 10 devices are activated for training. The results are shown in Fig.~\ref{fig:lstm_balanced_nobyz}, \ref{fig:lstm_balanced_byz}, and \ref{fig:lstm_balanced_byz_2}. For the poisoned workers, the labels of the local training data are randomly permuted.

In general, we get similar results as CNN on CIFAR-10.

\begin{figure*}[htb!]
\centering
\includegraphics[width=0.9\textwidth]{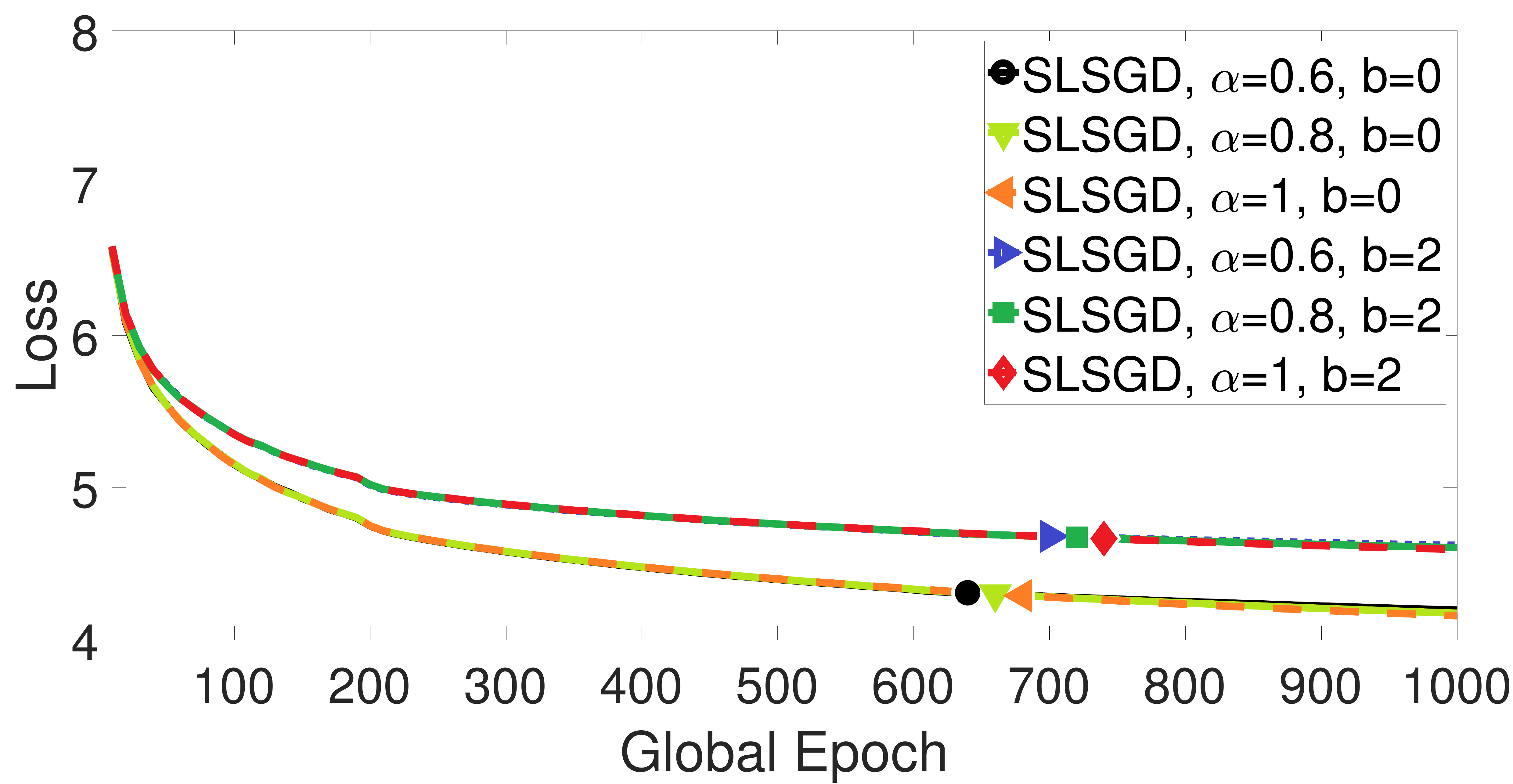}
\caption{LSTM experiment on WikiText-2. Convergence on training data with balanced partition, without attack. Each epoch is a full pass of the local training data. We take learning rate $\gamma=20$. Legend ``SLSGD, $\alpha=0.8, b=2$'' means that SLSGD takes $\trmean_2$ for aggregation, and the initial $\alpha=1$ decays by the factor of $0.8$ at the $600$th epoch. Note that SLSGD with $\alpha=1$ and $b=0$ is the baseline \textit{FedAvg}. Note that we fix the random seeds. Thus, before $\alpha$ decays at the $600$th epoch, results with the same $\gamma$ and $b$ are the same.}
\label{fig:lstm_balanced_nobyz}
\end{figure*}
\begin{figure*}[htb!]
\centering
\includegraphics[width=0.9\textwidth]{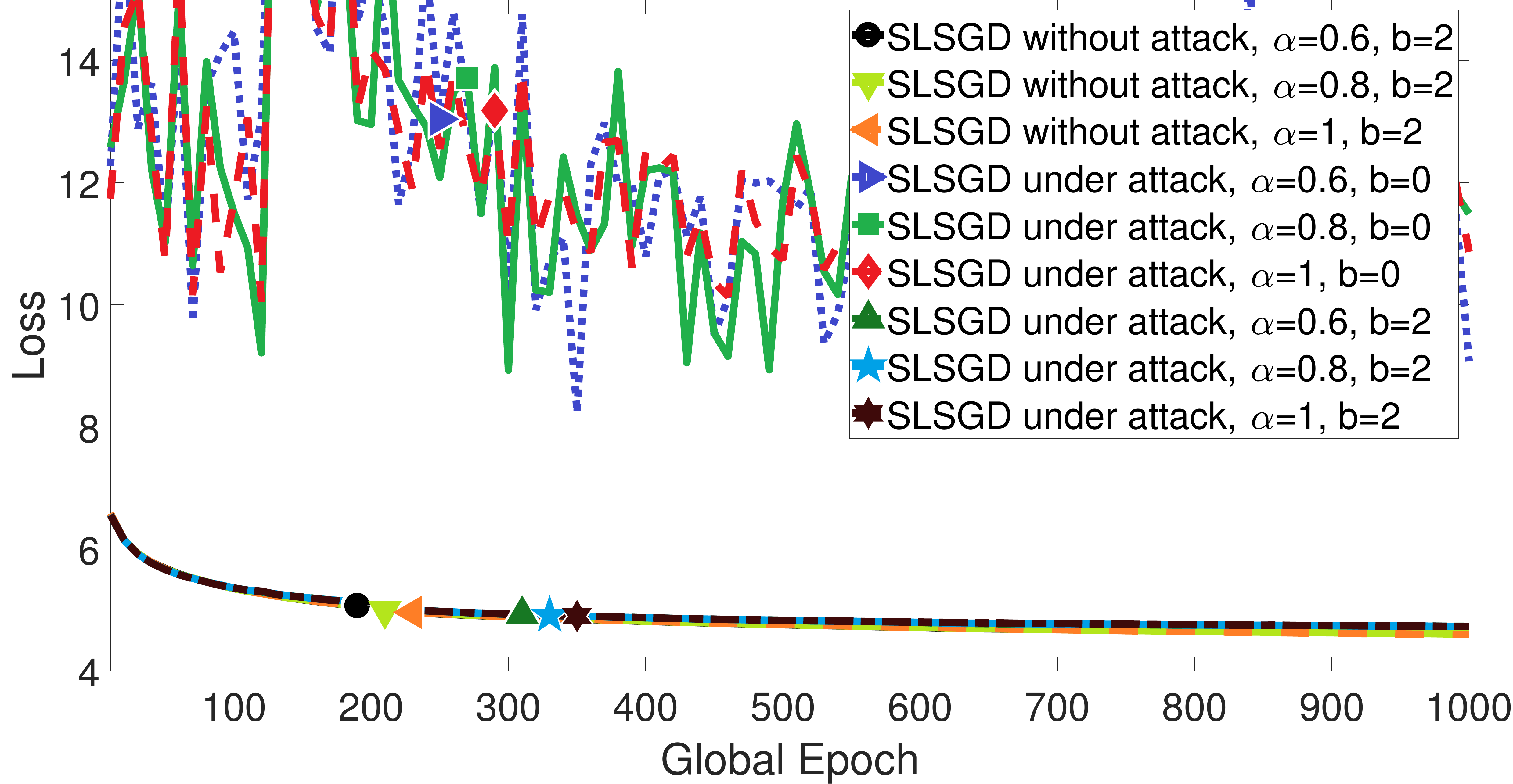}
\caption{LSTM experiment on WikiText-2. Convergence on training data with balanced partition, under attack. In each epoch, we guarantee that $q=2$ of the $k=10$ selected workers are training on poisoned data. Each epoch is a full pass of the local training data. We take learning rate $\gamma=20$. Legend ``SLSGD, $\alpha=0.8, b=2$'' means that SLSGD takes $\trmean_2$ for aggregation, and the initial $\alpha=1$ decays by the factor of $0.8$ at the $600$th epoch. Note that SLSGD with $\alpha=1$ and $b=0$ is the baseline \textit{FedAvg}. Note that we fix the random seeds. Thus, before $\alpha$ decays at the $600$th epoch, results with the same $\gamma$ and $b$ are the same.}
\label{fig:lstm_balanced_byz}
\end{figure*}
\begin{figure*}[htb!]
\centering
\includegraphics[width=0.99\textwidth]{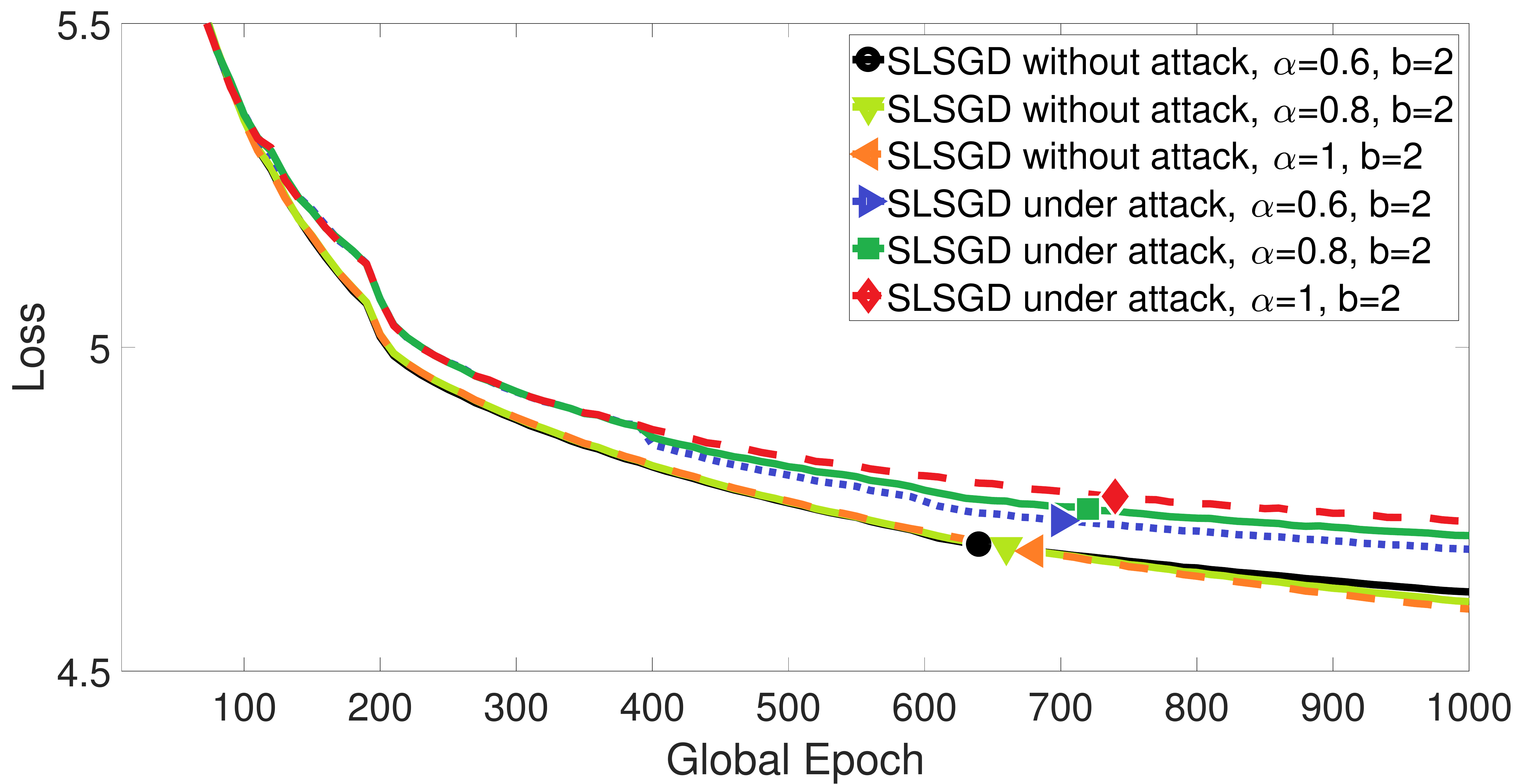}
\caption{(Zoomed) LSTM experiment on WikiText-2. Convergence on training data with balanced partition, under attack. In each epoch, we guarantee that $q=2$ of the $k=10$ selected workers are training on poisoned data. Each epoch is a full pass of the local training data. We take learning rate $\gamma=20$. Legend ``SLSGD, $\alpha=0.8, b=2$'' means that SLSGD takes $\trmean_2$ for aggregation, and the initial $\alpha=1$ decays by the factor of $0.8$ at the $600$th epoch. Note that SLSGD with $\alpha=1$ and $b=0$ is the baseline \textit{FedAvg}. Note that we fix the random seeds. Thus, before $\alpha$ decays at the $600$th epoch, results with the same $\gamma$ and $b$ are the same.}
\label{fig:lstm_balanced_byz_2}
\end{figure*}

\newpage
\section{Proofs}

\setcounter{lemma}{0}
\setcounter{theorem}{0}
\setcounter{corollary}{0}

\begin{theorem}
We take $\gamma \leq \min \left(\frac{1}{L}, 2\right)$. After $T$ epochs, Algorithm~\ref{alg:robust_fed} with Option I converges to a global optimum:
\begin{align*}
    &\E \left[ F(x_T) - F(x_*) \right] 
    \leq \left( 1-\alpha + \alpha (1 - \frac{\gamma}{2})^{H_{min}} \right)^T \left[ F(x_0) - F(x_*) \right] \\
    &\quad + \left[ 1 -  \left( 1-\alpha + \alpha (1 - \frac{\gamma}{2})^{H_{min}} \right)^T \right] \mathcal{O}\left( V_1 + \left( 1 + \frac{1}{k} - \frac{1}{n} \right) V_2 \right),
\end{align*}
where $V_2 = \max_{t \in \{0, T-1\}, h \in \{0, H^i_t - 1\}, i \in [n]} \|x_{t, h}^i - x_*\|^2$.
\end{theorem}

\begin{proof}
For convenience, we ignore the random sample $z \sim \mathcal{D}^i$ in our notations. Thus, $f^i(x_{t,h}^i)$ represents $f(x_{t, h}^i; z_{t, h}^i)$, where $z_{t, h}^i \sim \mathcal{D}^i$.
Furthermore, we define $F^i(x) = \E_{z \sim \mathcal{D}^i} f(x; z)$.

Thus, Line 10 in Algorithm~\ref{alg:robust_fed} can be rewritten into
\begin{align*}
x_{t, h}^i = x_{t, h-1}^i - \gamma \nabla f^i_t(x_{t, h-1}^i).
\end{align*}
Using $L$-smoothness of $F(x)$, we have 
\begin{align*}
& F(x_{t, h}^i)  \\
&\leq F(x_{t, h-1}^i) - \gamma \ip{\nabla F(x_{t, h-1}^i}{\nabla f^i(x_{t, h-1}^i)} + \frac{L\gamma^2}{2} \left\| \nabla f^i(x_{t, h-1}^i) \right\|^2 \\
&\leq F(x_{t, h-1}^i) - \gamma \ip{\nabla F(x_{t, h-1}^i}{\nabla f^i(x_{t, h-1}^i)} + \frac{\gamma}{2} \left\| \nabla f^i(x_{t, h-1}^i) \right\|^2 \\
&\leq F(x_{t, h-1}^i) - \frac{\gamma}{2} \left\| \nabla F(x_{t, h-1}^i) \right\|^2 + \frac{\gamma}{2} \left\| \nabla F(x_{t, h-1}^i) - f^i(x_{t, h-1}^i) \right\|^2.
\end{align*}

It is easy to check that $\exists \rho \geq 0$, $G_*(x) = F(x) + \frac{\rho}{2} \|x - x_*\|^2$ is $(\rho + \mu)$-strongly convex, where $\rho + \mu \geq 1$.
Thus, we have 
\begin{align*}
    F(x) - F(x_*) \leq G_*(x) - G_*(x_*) \leq \frac{\|\nabla G_*(x)\|^2}{2(\rho + \mu)} 
    \leq \frac{\|\nabla F(x)\|^2 + \rho^2 \|x - x_*\|^2}{\rho + \mu}.
\end{align*}

Taking expectation on both sides, conditional on $x_{t, h-1}^i$, we have
\begin{align*}
& \E\left[ F(x_{t, h}^i) - F(x_*) \right] \\
&\leq F(x_{t, h-1}^i) - F(x_*) - \frac{\gamma}{2} \left\| \nabla F(x_{t, h-1}^i) \right\|^2 + \gamma V_1 \\
&\leq F(x_{t, h-1}^i) - F(x_*) - \frac{\gamma}{2} (\rho + \mu) \left[ F(x_{t, h-1}^i) - F(x_*) \right] + \frac{\gamma \rho^2}{2} \|x_{t, h-1}^i - x_*\|^2 + \gamma V_1 \\
&\leq (1 - \frac{\gamma}{2}) \left[ F(x_{t, h-1}^i) - F(x_*) \right] + \gamma \mathcal{O}(V_1 + V_2).
\end{align*}

By telescoping and taking total expectation, we have
\begin{align*}
    &\E\left[ F(x_{t, H^i_t}^i) - F(x_*) \right] \\
    &\leq (1 - \frac{\gamma}{2})^{H^i_t} \left[ F(x_{t, 0}^i) - F(x_*) \right] + \frac{1 - (1 - \frac{\gamma}{2})^{H^i_t}}{1 - (1 - \frac{\gamma}{2})} \gamma \mathcal{O}(V_1 + V_2) \\
    &\leq (1 - \frac{\gamma}{2})^{H_{min}} \left[ F(x_{t, 0}^i) - F(x_*) \right] + \mathcal{O}(V_1 + V_2).
\end{align*}

$S_t \subseteq [n]$ is a subset of devices randomly sampled from all the $n$ devices without placement. Thus, we have
\begin{align*}
    &F\left( \frac{1}{k} \sum_{i \in S_t} x_{t, H^i_t}^i  \right) \\
    &\leq F\left( \frac{1}{n} \sum_{i \in [n]} x_{t, H^i_t}^i  \right) 
     + \ip{\nabla F\left( \frac{1}{n} \sum_{i \in [n]} x_{t, H^i_t}^i  \right)}{\frac{1}{k} \sum_{i \in S_t} x_{t, H^i_t}^i - \frac{1}{n} \sum_{i \in [n]} x_{t, H^i_t}^i} \\
    &\quad + \frac{L}{2} \left\| \frac{1}{k} \sum_{i \in S_t} x_{t, H^i_t}^i - \frac{1}{n} \sum_{i \in [n]} x_{t, H^i_t}^i \right\|^2 \\
    &\leq F\left( \frac{1}{n} \sum_{i \in [n]} x_{t, H^i_t}^i  \right)
     + \frac{1}{2} \left\| \nabla F\left( \frac{1}{n} \sum_{i \in [n]} x_{t, H^i_t}^i  \right) \right\|^2 
     + \frac{L+1}{2} \left\| \frac{1}{k} \sum_{i \in S_t} x_{t, H^i_t}^i - \frac{1}{n} \sum_{i \in [n]} x_{t, H^i_t}^i \right\|^2 \\
    &\leq F\left( \frac{1}{n} \sum_{i \in [n]} x_{t, H^i_t}^i  \right) + \mathcal{O}(V_1) + \left( \frac{1}{k} - \frac{1}{n} \right) \frac{1}{n-1} \sum_{i \in [n]} \frac{L+1}{2} \left\| x_{t, H^i_t}^i - \frac{1}{n} \sum_{j \in [n]} x_{t, H^j_t}^j \right\|^2 \\
    &\leq F\left( \frac{1}{n} \sum_{i \in [n]} x_{t, H^i_t}^i  \right) + \mathcal{O}(V_1) + \left( \frac{1}{k} - \frac{1}{n} \right) \mathcal{O}(V_2).
\end{align*}

On the server, after aggregation, conditional on $x_{t-1}$, we have 
\begin{align*}
    &\E \left[ F(x'_t) - F(x_*) \right] \\
    &\leq \E \left[ G_*(\frac{1}{n} \sum_{i \in [n]} x_{t, H^i_t}^i)  - F(x_*) \right] + \E \left[ F\left( \frac{1}{k} \sum_{i \in S_t} x_{t, H^i_t}^i  \right) - F\left( \frac{1}{n} \sum_{i \in [n]} x_{t, H^i_t}^i  \right) \right]\\
    &\leq \E \left[ \frac{1}{n} \sum_{i \in [n]} G_*( x_{t, H^i_t}^i )  - F(x_*) \right] + \mathcal{O}(V_1) + \left( \frac{1}{k} - \frac{1}{n} \right) \mathcal{O}(V_2)\\
    &\leq \E \left[ \frac{1}{n} \sum_{i \in [n]} F( x_{t, H^i_t}^i ) - F(x_*) \right] + \mathcal{O}\left( V_1 + \left( 1 + \frac{1}{k} - \frac{1}{n} \right) V_2 \right)  \\
    &\leq \frac{1}{n} \sum_{i \in [n]} (1 - \frac{\gamma}{2})^{H_{min}} \left[ F(x_{t, 0}^i) - F(x_*) \right] + \mathcal{O}\left( V_1 + \left( 1 + \frac{1}{k} - \frac{1}{n} \right) V_2 \right) \\
    &\leq (1 - \frac{\gamma}{2})^{H_{min}} \left[ F(x_{t-1}) - F(x_*) \right] + \mathcal{O}\left( V_1 + \left( 1 + \frac{1}{k} - \frac{1}{n} \right) V_2 \right). 
\end{align*}

We define $G_{t-1}(x) = F(x) + \frac{\rho}{2} \|x - x_{t-1}\|^2$, which is convex. Then, we have
\begin{align*}
    &\E \left[ F(x_t) - F(x_*) \right] \\
    &\leq \E \left[ G_{t-1}(x_t) - F(x_*) \right] \\
    &\leq \E \left[ (1-\alpha) G_{t-1}(x_{t-1}) + \alpha G_{t-1}(x'_{t}) - F(x_*) \right] \\
    &\leq (1-\alpha) \left[ F(x_{t-1}) - F(x_*) \right] + \alpha \E\left[ F(x'_t) - F(x_*) + \frac{\rho}{2} \|x'_t - x_{t-1}\|^2\right] \\
    &\leq (1-\alpha) \left[ F(x_{t-1}) - F(x_*) \right] + \alpha (1 - \frac{\gamma}{2})^{H_{min}} \left[ F(x_{t-1}) - F(x_*) \right] + \alpha \mathcal{O}(V_1 + V_2)\\
    &\leq \left[ 1-\alpha + \alpha (1 - \frac{\gamma}{2})^{H_{min}} \right] \left[ F(x_{t-1}) - F(x_*) \right] + \alpha \mathcal{O}\left( V_1 + \left( 1 + \frac{1}{k} - \frac{1}{n} \right) V_2 \right).
\end{align*}

After $T$ epochs, by telescoping and taking total expectation, we have
\begin{align*}
    &\E \left[ F(x_T) - F(x_*) \right] \\
    &\leq \left[ 1-\alpha + \alpha (1 - \frac{\gamma}{2})^{H_{min}} \right]^T \left[ F(x_0) - F(x_*) \right] \\
    &\quad + \left[ 1 -  \left( 1-\alpha + \alpha (1 - \frac{\gamma}{2})^{H_{min}} \right)^T \right] \mathcal{O}\left( V_1 + \left( 1 + \frac{1}{k} - \frac{1}{n} \right) V_2 \right).
\end{align*}

\end{proof}

\begin{theorem}
Assume that additional to the $n$ normal workers, there are $q$ workers training on poisoned data, where $q \ll n$, and $2q \leq 2b < k$. 
We take $\gamma \leq \min \left(\frac{1}{L}, 2\right)$. After $T$ epochs, Algorithm~\ref{alg:robust_fed} with Option II converges to a global optimum:
\begin{align*}
    &\E \left[ F(x_T) - F(x_*) \right] 
    \leq \left( 1-\alpha + \alpha (1 - \frac{\gamma}{2})^{H_{min}} \right)^T \left[ F(x_0) - F(x_*) \right] \\
    & \quad + \left[ 1 -  \left( 1-\alpha + \alpha (1 - \frac{\gamma}{2})^{H_{min}} \right)^T \right] \left[  \mathcal{O}(\beta V_2) + \mathcal{O}(V_1) \right],
\end{align*}
where $V_2 = \max_{t \in \{0, T-1\}, h \in \{0, H^i_t - 1\}, i \in [n]} \|x_{t, h}^i - x_*\|^2$, $\beta = 1 + \frac{1}{k-q} - \frac{1}{n} + \frac{k (k+b)}{(k-b-q)^2}$.
\end{theorem}

\begin{proof}
First, we analyze the robustness of trimmed mean. 
Assume that among the scalar sequence $\{\tilde{v}_i: i \in [k]\}$, $q_1$ elements are poisoned. Without loss of generality, we denote the remaining correct values as $\{v_1, \ldots, v_{k-q_1}\}$. Thus, for $q_1 < b \leq \lceil k/2 \rceil - 1$, $v_{(b-q_1+i):(k-q_1)} \leq \tilde{v}_{(b+i):k} \leq v_{(b+i):(k-q_1)}$, for $\forall i \in [k-2b]$, where $\tilde{v}_{(b+i):k}$ is the $(b+i)$th smallest element in $\{\tilde{v}_i: i \in [k]\}$, and $v_{(b+i):(k-q_1)}$ is the $(b+i)$th smallest element in $\{v_1, \ldots, v_{k-q_1}\}$.

Define $\bar{v} = \frac{1}{k - q_1} \sum_{i \in [k-q_1]} v_i$. We have

\begin{align*}
\quad&\sum_{i=b-q_1+1}^{k-q_1-b} (v_{i:(k-q_1)}-\bar{v}) \leq \sum_{i=b+1}^{k-b} (\tilde{v}_{i:k}-\bar{v}) \leq \sum_{i=b+1}^{k-b} (v_{i:(k-q_1)}-\bar{v}) \\
\Rightarrow & \frac{\sum_{i=1}^{k-q_1-b} (v_{i:(k-q_1)}-\bar{v})}{k-b-q_1} \leq \frac{\sum_{i=b+1}^{k-b} (\tilde{v}_{i:k}-\bar{v})}{k-2b} \leq \frac{\sum_{i=b+1}^{k-q_1} (v_{i:(k-q_1)}-\bar{v})}{k-b-q_1} \\
\Rightarrow& \left[ \frac{\sum_{i=b+1}^{k-b} (\tilde{v}_{i:k}-\bar{v})}{k-2b} \right]^2 \leq \max \left\{\left[ \frac{\sum_{i=1}^{k-q_1-b} (v_{i:(k-q_1)}-\bar{v})}{k-b-q_1} \right]^2, \left[ \frac{\sum_{i=b+1}^{k-q_1} (v_{i:(k-q_1)}-\bar{v})}{k-b-q_1} \right]^2 \right\}.
\end{align*}
Thus, we have 
\begin{align*}
&\left[ \trmean_b(\{\tilde{v}_i: i \in [k]\}) - \bar{v} \right]^2 \\
&= \left[ \frac{\sum_{i=b+1}^{k-b} \tilde{v}_{i:k}}{k-2b} - \bar{v} \right]^2 \\
&\leq \max \left\{\left[ \frac{\sum_{i=1}^{k-q_1-b} (v_{i:(k-q_1)}-\bar{v})}{k-b-q_1} \right]^2, \left[ \frac{\sum_{i=b+1}^{k-q_1} (v_{i:(k-q_1)}-\bar{v})}{k-b-q_1} \right]^2 \right\}.
\end{align*}
Note that for arbitrary subset $\mathcal{S} \subseteq [k-q_1]$ with cardinality $|\mathcal{S}| = k-b-q_1$, we have the following bound:
\begin{align*}
&\left[ \frac{\sum_{i \in \mathcal{S}} (v_{i:(k-q_1)}-\bar{v})}{k-b-q_1 } \right]^2 \\
&= \left[ \frac{\sum_{i \in [k-q_1]} (v_{i:(k-q_1)}-\bar{v}) - \sum_{i \notin \mathcal{S}} (v_{i:(k-q_1)}-\bar{v})}{k-b-q_1} \right]^2 \\
&\leq 2\left[ \frac{\sum_{i \in [k-q_1]} (v_{i:(k-q_1)}-\bar{v})}{k-b-q_1} \right]^2 + 2\left[ \frac{\sum_{i \notin \mathcal{S}} (v_{i:(k-q_1)}-\bar{v})}{k-b-q_1} \right]^2 \\
&= \frac{2(k-q_1)^2}{(k-b-q_1)^2} \left[ \frac{\sum_{i \in [k-q_1]} (v_{i:(k-q_1)}-\bar{v})}{k-q_1} \right]^2 + \frac{2b^2}{(k-b-q_1)^2} \left[ \frac{\sum_{i \notin \mathcal{S}} (v_{i:(k-q_1)}-\bar{v})}{b} \right]^2 \\
&\leq \frac{2(k-q_1)^2}{(k-b-q_1)^2} \left[ \frac{\sum_{i \in [k-q_1]} (v_{i:(k-q_1)}-\bar{v})}{k-q_1} \right]^2 
+ \frac{2b^2}{(k-b-q_1)^2} \frac{\sum_{i \notin \mathcal{S}} (v_{i:(k-q_1)}-\bar{v})^2}{b}  \\
&\leq \frac{2(k-q_1)^2}{(k-b-q_1)^2} \left[ \frac{\sum_{i \in [k-q_1]} (v_{i:(k-q_1)}-\bar{v})}{k-q_1} \right]^2 
+ \frac{2b^2}{(k-b-q_1)^2} \frac{\sum_{i \in [k-q_1]} (v_{i:(k-q_1)}-\bar{v})^2}{b} \\
&\leq \frac{2(k-q_1)^2}{(k-b-q_1)^2}  \frac{\sum_{i \in [k-q_1]} (v_{i:(k-q_1)}-\bar{v})^2}{k-q_1} 
+ \frac{2b^2}{(k-b-q_1)^2} \frac{\sum_{i \in [k-q_1]} (v_{i:(k-q_1)}-\bar{v})^2}{b} \\
&\leq \frac{2(k-q_1) (k+b-q_1)}{(k-b-q_1)^2}  \frac{\sum_{i \in [k-q_1]} (v_{i:(k-q_1)}-\bar{v})^2}{k-q_1}.
\end{align*}

In the worker set $S_t$, there are $q_1$ poisoned workers. We denote $C_t \subseteq S_t$ as the set of normal workers with cardinality $|C_t| = k-q_1$.

Thus, we have
\begin{align*}
&\left\| \trmean_b(\{x^i_{t,H^i_t}: i \in S_t\}) - \frac{1}{k-q_1} \sum_{i \in C_t} x^i_{t,H^i_t} \right\|^2 \\
&\leq \frac{2(k-q_1) (k+b-q_1)}{(k-b-q_1)^2}  \frac{\sum_{i \in C_t} \left\| x^i_{t,H^i_t} - \frac{1}{k-q_1} \sum_{i \in C_t} x^i_{t,H^i_t} \right\|^2}{k-q_1} \\
&\leq \frac{(k-q_1) (k+b-q_1)}{(k-b-q_1)^2}  \mathcal{O}(V_2).
\end{align*}

Using $L$-smoothness, we have
\begin{align*}
    &F\left( \trmean_b(\{x^i_{t,H^i_t}: i \in S_t\}) \right) \\
    &\leq F\left( \frac{1}{k-q_1} \sum_{i \in C_t} x^i_{t,H^i_t} \right) \\
    &\quad + \ip{\nabla F\left( \frac{1}{k-q_1} \sum_{i \in C_t} x^i_{t,H^i_t} \right)}{\trmean_b(\{x^i_{t,H^i_t}: i \in S_t\}) - \frac{1}{k-q_1} \sum_{i \in C_t} x^i_{t,H^i_t}} \\
    &\quad +\frac{L}{2} \left\| \trmean_b(\{x^i_{t,H^i_t}: i \in S_t\}) - \frac{1}{k-q_1} \sum_{i \in C_t} x^i_{t,H^i_t} \right\|^2 \\
    &\leq F\left( \frac{1}{k-q_1} \sum_{i \in C_t} x^i_{t,H^i_t} \right) + 2\left\| \nabla F\left( \frac{1}{k-q_1} \sum_{i \in C_t} x^i_{t,H^i_t} \right) \right\|^2 \\
    &\quad  + \left(\frac{L}{2} + 2 \right)
    \left\| \trmean_b(\{x^i_{t,H^i_t}: i \in S_t\}) - \frac{1}{k-q_1} \sum_{i \in C_t} x^i_{t,H^i_t} \right\|^2 \\
    &\leq F\left( \frac{1}{k-q_1} \sum_{i \in C_t} x^i_{t,H^i_t} \right) + \frac{k (k+b)}{(k-b-q)^2}  \mathcal{O}(V_2) + \mathcal{O}(V_1).
\end{align*}

Combining with Theorem~\ref{thm:convergence}, on the server, after aggregation using trimmed mean, conditional on $x_{t-1}$, we have 
\begin{align*}
    &\E \left[ F(x'_t) - F(x_*) \right] \\
    &=\E \left[ F\left( \trmean_b(\{x^i_{t,H^i_t}: i \in S_t\}) \right) - F(x_*) \right] \\
    &=\E \left[ F\left( \frac{1}{k-q_1} \sum_{i \in C_t} x^i_{t,H^i_t} \right) + F\left( \trmean_b(\{x^i_{t,H^i_t}: i \in S_t\}) \right) - F\left( \frac{1}{k-q_1} \sum_{i \in C_t} x^i_{t,H^i_t} \right) - F(x_*) \right] \\
    &\leq \E \left[ F\left( \frac{1}{k-q_1} \sum_{i \in C_t} x^i_{t,H^i_t} \right)  - F(x_*) \right] + \frac{k (k+b)}{(k-b-q)^2}  \mathcal{O}(V_2) + \mathcal{O}(V_1) \\
    &\leq \E \left[ G_*\left( \frac{1}{n} \sum_{i \in [n]} x^i_{t,H^i_t} \right)  - F(x_*) \right] + \left( 1 + \frac{1}{k-q} - \frac{1}{n} + \frac{k (k+b)}{(k-b-q)^2} \right)    \mathcal{O}(V_2) + \mathcal{O}(V_1)\\
    &\leq \E \left[ \frac{1}{n} \sum_{i \in [n]} G_*( x_{t, H^i_t}^i )  - F(x_*) \right] + \left( 1 + \frac{1}{k-q} - \frac{1}{n} + \frac{k (k+b)}{(k-b-q)^2} \right)    \mathcal{O}(V_2) + \mathcal{O}(V_1)\\
    &\leq \E \left[ \frac{1}{n} \sum_{i \in [n]} F( x_{t, H^i_t}^i ) - F(x_*) \right] + \left( 1 + \frac{1}{k-q} - \frac{1}{n} + \frac{k (k+b)}{(k-b-q)^2} \right)    \mathcal{O}(V_2) + \mathcal{O}(V_1)\\
    &\leq \frac{1}{n} \sum_{i \in [n]} (1 - \frac{\gamma}{2})^{H_{min}} \left[ F(x_{t, 0}^i) - F(x_*) \right] + \left( 1 + \frac{1}{k-q} - \frac{1}{n} + \frac{k (k+b)}{(k-b-q)^2} \right)    \mathcal{O}(V_2) + \mathcal{O}(V_1)\\
    &\leq (1 - \frac{\gamma}{2})^{H_{min}} \left[ F(x_{t-1}) - F(x_*) \right] + \left( 1 + \frac{1}{k-q} - \frac{1}{n} + \frac{k (k+b)}{(k-b-q)^2} \right)    \mathcal{O}(V_2) + \mathcal{O}(V_1). 
\end{align*}

After $T$ epochs, by telescoping and taking total expectation, we have
\begin{align*}
    &\E \left[ F(x_T) - F(x_*) \right] \\
    &\leq \left[ 1-\alpha + \alpha (1 - \frac{\gamma}{2})^{H_{min}} \right]^T \left[ F(x_0) - F(x_*) \right] \\
    &\quad + \left[ 1 -  \left[ 1-\alpha + \alpha (1 - \frac{\gamma}{2})^{H_{min}} \right]^T \right] \left[ \left( 1 + \frac{1}{k-q} - \frac{1}{n} + \frac{k (k+b)}{(k-b-q)^2} \right)    \mathcal{O}(V_2) + \mathcal{O}(V_1) \right].
\end{align*}

\end{proof}

\end{document}